\documentclass{article}

\usepackage{microtype}
\usepackage{graphicx}
\usepackage{subfigure}
\usepackage{booktabs} 

\usepackage{hyperref}

\usepackage[accepted]{icml2024}

\usepackage{amsmath}
\usepackage{amssymb}
\usepackage{mathtools}
\usepackage{amsthm}

\usepackage[capitalize,noabbrev]{cleveref}

\theoremstyle{plain}
\newtheorem{theorem}{Theorem}[section]

\newtheorem{lemma}[theorem]{Lemma}
\newtheorem{corollary}[theorem]{Corollary}
\theoremstyle{definition}

\theoremstyle{remark}
\newtheorem{remark}[theorem]{Remark}

\usepackage[textsize=tiny]{todonotes}

\icmltitlerunning{Double Variance Reduction: A Smoothing Trick for Composite Optimization Problems without First-Order Gradient}

\usepackage{amsmath,amsfonts,bm, amsthm, cancel, xcolor, nccmath, mathtools, multirow}

\def\DM{{\mathcal D}}

\def\NM{{\mathcal N}}
\def\OM{{\mathcal O}}

\def\XM{{\mathcal X}}

\def\RB{{\mathbb R}}
\def\EB{{\mathbb E}}

\def\dif{\mathop{}\hphantom{\mskip-\thinmuskip}\mathrm{d}}%
\let\daccent\d
\let\d\relax
\newcommand\d{\ifmmode\dif\else\expandafter\daccent\fi}
\newcommand{\dotprod}[1]{\left\langle #1\right\rangle}
\newcommand{\norm}[1]{\Vert #1\Vert^2}
\newcommand{\bnorm}[1]{\Big\Vert #1\Big\Vert^2}
\newcommand{\td}[1]{\tilde{#1}}
\newcommand{\hn}{\hat{\nabla}}
\newcommand{\dom}{\RB^d}

\DeclareMathOperator*{\argmin}{arg\,min}
\DeclareMathOperator{\prox}{\mathrm{prox}_{\eta\psi}}
\DeclareMathOperator{\tr}{tr}

\newlength{\leftstackrelawd}
\newlength{\leftstackrelbwd}
\def\leftstackrel#1#2{\settowidth{\leftstackrelawd}%
{${{}^{#1}}$}\settowidth{\leftstackrelbwd}{$#2$}%
\addtolength{\leftstackrelawd}{-\leftstackrelbwd}%
\leavevmode\ifthenelse{\lengthtest{\leftstackrelawd>0pt}}%
{\kern-.5\leftstackrelawd}{}\mathrel{\mathop{#2}\limits^{#1}}}

\begin{document}

\twocolumn[
\icmltitle{Double Variance Reduction: A Smoothing Trick for Composite Optimization Problems without First-Order Gradient}



\icmlsetsymbol{equal}{*}

\begin{icmlauthorlist}
\icmlauthor{Hao Di}{XJTU,work}
\icmlauthor{Haishan Ye}{XJTU,SGIT}
\icmlauthor{Yueling Zhang}{bfsu}
\icmlauthor{Xiangyu Chang}{XJTU}
\icmlauthor{Guang Dai}{SGIT}
\icmlauthor{Ivor W. Tsang}{Astar,NTU}
\end{icmlauthorlist}

\icmlaffiliation{XJTU}{Center for Intelligent Decision-Making and Machine Learning, School of Management, Xi’an Jiaotong University, China.}
\icmlaffiliation{work}{This work was completed during the internship at SGIT AI Lab, State Grid Corporation of China.}
\icmlaffiliation{SGIT}{SGIT AI Lab, State Grid Corporation of China.}
\icmlaffiliation{Astar}{CFAR and IHPC, Agency for Science, Technology and Research (A*STAR), Singapore}
\icmlaffiliation{NTU}{College of Computing and Data Science, NTU, Singapore.}
\icmlaffiliation{bfsu}{International Business School, Beijing Foreign Studies University, Beijing, China.}

\icmlcorrespondingauthor{Haishan Ye}{yehaishan@xjtu.edu.cn}

\icmlkeywords{Zeroth-Order Optimization, Variance Reduction}

\vskip 0.3in
]



\printAffiliationsAndNotice{}  

\begin{abstract}
Variance reduction techniques are designed to decrease the sampling variance, thereby accelerating convergence rates of first-order (FO) and zeroth-order (ZO) optimization methods.
However, in composite optimization problems, ZO methods encounter an additional variance called the coordinate-wise variance, which stems from the random gradient estimation.
To reduce this variance, prior works require estimating all partial derivatives, essentially approximating FO information.
This approach demands $\OM(d)$ function evaluations ($d$ is the dimension size), which incurs substantial computational costs and is prohibitive in high-dimensional scenarios. 
This paper proposes the Zeroth-order Proximal Double Variance Reduction (\texttt{ZPDVR}) method, which utilizes the averaging trick to reduce both sampling and coordinate-wise variances.
Compared to prior methods, \texttt{ZPDVR} relies solely on random gradient estimates, calls the stochastic zeroth-order oracle (SZO) in expectation $\OM(1)$ times per iteration, and achieves the optimal $\OM(d(n + \kappa)\log (\frac{1}{\epsilon}))$ SZO query complexity in the strongly convex and smooth setting, where $\kappa$ represents the condition number and $\epsilon$ is the desired accuracy.
Empirical results validate \texttt{ZPDVR}'s linear convergence and demonstrate its superior performance over other related methods.
\end{abstract}

\section{Introduction}
\label{sec:introduction}

This paper considers the following composite finite-sum optimization problem,
\begin{equation}
    \begin{aligned}
    \label{eq:objective_function}
    \min_{x\in \dom} \quad & F(x) = f(x) + \psi(x), \\
    \text{with} \quad & f(x) = \frac{1}{n}\sum_{i=1}^n f_i(x),
    \end{aligned}
\end{equation}
where $f_i(x): \RB^d\to \RB$ is a $\mu$-strongly convex and $L$-smooth function
, and $\psi(x):\RB^d \to \RB$ is a convex and non-smooth function, such as the $\ell_1$ regularization.
This formulation is prevalent in several critical applications, such as signal compression \citep{jenatton2011proximal}, image recovery \citep{chen2021deep}, sparse model training \citep{beck2009fast, yun2021adaptive}.
In this study, we assume that only the zeroth-order (ZO) oracle (i.e., the value of function $f(x)$) is available.
Under this condition, we explore the ZO variance reduction method to achieve the linear convergence rate for the above composite optimization problem \eqref{eq:objective_function} with only random gradient estimates.

\begin{table*}[t]
\caption{Summary of SZO query complexity and expected SZO queries per iteration of different methods. FO and ZO represent first-order and zeroth-order methods, respectively. PGD is the proximal gradient descent method.}
\label{tab:complexity result}
\vskip 0.15in
\begin{center}
\begin{small}
\begin{tabular}{cllc}
\toprule
                              & Methods  & SZO query complexity & Expected SZO queries per iteration \\ \midrule
                              
\multirow{2}{*}{FO }           & PGD     &        $\mathcal{O}(dn\kappa\log(\frac{1}{\epsilon}))$        & $\mathcal{O}(nd)$ \\        
                             & PSVRG \citep{xiao2014proximal}   &       $\mathcal{O}(d(n+\kappa)\log(\frac{1}{\epsilon}))$       & $\mathcal{O}(d)$ \\ \midrule
\multirow{2}{*}{ZO }         & SEGA \citep{hanzely2018sega}  &        $\mathcal{O}(dn\kappa\log(\frac{1}{\epsilon}))$       & $\mathcal{O}(n)$ \\
                              & ZPDVR (ours)    &        $\OM(d(n+\kappa)\log(\frac{1}{\epsilon}))$        & $\mathcal{O}(1)$  \\ \bottomrule
\end{tabular}
\end{small}
\end{center}
\vskip -0.1in
\end{table*}

ZO optimization methods approximate the directional derivative information by utilizing either the function value or the difference in function values, and employ this estimate to minimize the objective function \citep{liu2020primer}.
Due to their exclusive reliance on function evaluations, ZO optimization methods simplify algorithm implementation \citep{nesterov2017random, liu2020primer} and are increasingly employed in scenarios where gradient information is either inaccessible or impractical, such as in neural network black-box attacks and robotic stiffness control \citep{tu2019autozoom, li2023stochastic}. 
Recently, \citet{malladi2023finetuning} have capitalized on the benefits of ZO optimization methods for fine-tuning large language models, achieving a comparable performance to first-order (FO) methods while significantly reducing memory usage. 
Furthermore, \citet{chen2023deepzero} employ ZO methods to train the ResNet-20 model from scratch, illustrating the practicality of these methods in large-scale machine learning applications.

To decrease the variance resulting from the random sampling and accelerate the convergence rate of ZO optimization methods for finite-sum optimization problems (i.e., $\min_{x\in \RB^d} \sum_{i=1}^n f_i(x)$), researchers have introduced various variance reduction techniques tailored for diverse settings, encompassing non-convex and non-convex non-smooth problems \citep{fang2018spider, ji2019improved,chen23ai, kazemi2023efficient}, spanning from centralized to decentralized scenarios \citep{yi2022zeroth, lin2023decentralized}.
In addition to the sampling variance, ZO gradient estimates also suffer from high inherent gradient approximation variance, which is $\OM(d)$ times the value of gradient norm $\norm{\nabla f(x)}$ \citep{nesterov2017random, liu2018zeroth}.
This variance stems from the random directional gradient estimation in ZO methods and can be eliminated by estimating each partial derivative along the coordinate directions instead of random directions.
Consequently, we refer to this variance as the \textit{coordinate-wise variance} throughout this paper.
For the unconstrained finite-sum optimization problem, the FO optimality condition indicates that the coordinate-wise variance will automatically diminish as $x\to x^*$.

However, in the context of the composite optimization problem \eqref{eq:objective_function} as well as the constrained optimization problem, a crucial characteristic is that $\norm{\nabla f(x^*)}$ is generally not equal to $0$ at the optimal point $x^*$.
Therefore, the coordinate-wise variance persists even after reaching the optimal point $x^*$.
Consequently, this non-vanishing variance induced by the random gradient estimate inhibits the convergence to the optimal point $x^*$, which leads to inferior performance.
Prior ZO variance reduction methods for the composite optimization problem \citep{liu2018zeroth, kazemi2023efficient} usually adopt stochastic variance reduction gradient (SVRG) techniques \citep{johnson2013accelerating, nguyen2017sarah, fang2018spider} to eliminate the inherent sampling variance and approximate the gradient by computing all partial gradient estimations to circumvent the coordinate-wise variance.
Deriving all partial gradient estimations requires $\Omega(d)$ function evaluations to make a step, which is prohibitive for high-dimensional problems and hinders the application of ZO variance reduction methods.

In this work, we employ the averaging trick and introduce a novel gradient estimator for the gradient $\nabla f(x)$, utilizing only random gradient estimates.
This proposed estimator calculates the difference between the random gradient estimate and the vector value of the current estimator along the same direction and uses the average scheme to refine estimation errors.
We demonstrate that this gradient estimator can progressively reduce the coordinate-wise variance.
Furthermore, by combining the proposed estimator with SVRG techniques, we present the \textbf{Z}eroth-order \textbf{P}roximal \textbf{D}ouble \textbf{V}ariance \textbf{R}eudction method, \texttt{ZPDVR}, which can reduce the inherent variances — both the random sampling variance and the coordinate-wise variance, in the ZO composite optimization problem, using only the random gradient estimate.
We also conduct a comprehensive analysis of \texttt{ZPDVR} and prove that the stochastic zeroth-order oracle (SZO) query complexity (the number of function estimations) of \texttt{ZPDVR} is $\OM(d(n+\kappa)\log(\frac{1}{\epsilon}))$ and it only calls $\OM(1)$ SZO in expectation per iteration.
As shown in \cref{tab:complexity result}, compared to other related FO and ZO methods, \texttt{ZPDVR} obtains the lowest SZO query complexity and exhibits the fewest SZO calls in expectation per iteration, which underscores its computational efficiency. 
The main contributions of this work are summarized as follows:
\begin{itemize}
    \item We propose a novel ZO variance reduction method for the composite optimization problem \eqref{eq:objective_function} using only random gradient estimates, \texttt{ZPDVR}, which reduce the inherent variances induced by random sampling and random gradient estimation. 
    \texttt{ZPDVR} calls the SZO $\OM(1)$ times in expectation per iteration, which demonstrates efficiency in updating parameters per iteration and exhibits applicability in high-dimensional scenarios.
    \item We conduct a convergence analysis of \texttt{ZPDVR}, and derive the $\mathcal{O}(d(n+\kappa)\log\frac{1}{\epsilon})$ SZO query complexity for \texttt{ZPDVR}, which aligns with the best-know SZO query complexity of SVRG methods and requires fewer SZO function evaluations than other related FO and ZO methods. The optimal convergence property further demonstrates the computational efficiency of \texttt{ZPDVR}.
\end{itemize}
\textbf{Notation.} Let $x^*$ denote the optimal point of the function $F$, specifically $x^*=\argmin_{x\in \RB^d} F(x)$, and $[n]$ represent the set $\{1, \dots, n\}$. 
For the $\mu$-strongly convex and $L$-smooth function $f_i$, $i\in[n]$, there exist two constants $\mu, L>0$ and $\mu\leq L$ such that 
\begin{equation*}
    \begin{aligned}
    f_i(x_1) &\geq f_i(x_2) + \dotprod{\nabla f(x_2), x_1 - x_2} + \frac{\mu}{2}\norm{x_1 - x_2},\\
    f_i(x_1) &\leq f_i(x_2) + \dotprod{\nabla f(x_2), x_1 - x_2} + \frac{L}{2}\norm{x_1 - x_2},   
    \end{aligned}
\end{equation*}
for $\forall x_1, x_2 \in \RB^d$. For the convex function $\psi$, the following inequality holds for $\forall x_1, x_2 \in \RB^d$
\begin{align*}
\psi(x_1) \geq \psi(x_2) + \dotprod{\partial \psi(x_2), x_1 - x_2},
\end{align*}
where $\partial \psi(x)$ is the subgradient of $\psi$ at the point $x$.
The proximal operator, denoted as $\prox$, is defined as $\prox(x) = \argmin_{x'\in \dom} \psi(x) + \frac{\eta}{2}\norm{x-x'}$. 
$\kappa = \frac{L}{\mu}$ is the condition number of the function $f$, and $\eta$ denotes the constant learning rate.
If $f(x) = \OM(g(x))$, it implies there is a positive constant $M>0$ such that $\vert f(x)\vert \leq M g(x)$. 
$f(x) = \Omega(g(x))$ indicates that there is a positive constant $C$ such that $C \vert g(x)\vert \leq \vert f(x)\vert $.

\section{Related Work}
\label{sec:related_work}

\textbf{Variance Reduction.} 
Variance reduction methods leverage the control variate technique \citep{rubinstein1985efficiency} to reduce inherent sampling variance in stochastic gradient descent (SGD).
Specifically, for the strongly convex and smooth finite-sum optimization problem, FO variance reduction methods attain the best-know $\mathcal{O}((\kappa + n)\log(\frac{1}{\epsilon}))$ gradient evaluations \citep{gower2020variance} , resulting in a significant acceleration of $\OM(n\kappa\log(\frac{1}{\epsilon}))$ in SGD.
The classic SVRG method \citep{johnson2013accelerating} adopts a double-loop structure, maintaining a snapshot of model parameters to compute the full gradient in the outer loop and constructing an unbiased gradient estimate in the inner loop.
\citet{allen2016variance} and \citet{reddi2016stochastic} extend SVRG into the non-convex setting and derive an improved query complexity than gradient descent (GD) by a factor of $\OM(n^{\frac{1}{3}})$.
Subsequent work \citep{fang2018spider} leverages the stochastic path-integrated differential estimator, proposing SPIDER with $\OM(n^{\frac{1}{2}}\epsilon)$ query complexity for non-convex finite-sum optimization problems.
Diverging from the double-loop structure, \citet{kovalev2020don} and \citet{li2021page} propose loop-less SVRG methods that execute the coin-flip process independently of prior knowledge about the condition number. 
In cases where only a gradient sketch is available, \citet{hanzely2018sega} and \citet{gower2021stochastic} develop gradient estimators via a sketch-and-project process, progressively diminishing the estimation variance.

\textbf{Zeroth-order Optimization.} ZO optimization methods leverage the function evaluation to approximate the gradient and utilize this estimate to minimize the objective function. 
\citet{nesterov2017random} introduce the Gaussian smooth technique and achieve the $\OM(dn\kappa\log (\frac{1}{\epsilon}))$ SZO query complexity for ZO-GD in the strongly convex and smooth scenario.
The convergence properties of ZO-SGD for non-convex problems are analyzed by \citet{ghadimi2013stochastic}.
\citet{ye2018hessian} integrate Hessian information into gradient estimation and propose a Hessian-aware ZO optimization method.
\citet{berahas2022theoretical} derive bounds on the number of samples and the sampling radius to ensure the convergence of ZO-GD with different ZO gradient estimators.
Compared to the FO methods, the query complexity of ZO methods exhibits a dependence on the dimension size $d$ \citep{liu2020primer}.

\textbf{Zeroth-order Variance Reduction.} Combining with SVRG, \citet{liu2018zeroth} introduce the ZO-SVRG method for non-convex finite-sum problems.
Subsequently, \citet{ji2019improved} improve the SZO query complexity of ZO-SVRG, surpassing both ZO-GD and ZO-SGD in the non-convex finite-sum setting and aligning with the results of FO SVRG methods.
For the composite optimization problem \eqref{eq:objective_function}, \citet{huang2019faster} present the convergence property of ZO proximal SVRG (ZPSVRG) method for non-convex non-smooth problems, and demonstrate that only utilizing the random gradient estimate incurs additional bias and the FO gradient information is necessary.
\citet{kazemi2023efficient} improve the iterative complexity of ZPSVRG for the non-convex non-smooth problem.

\section{Methodology}
\label{sec:methodology}
This section presents details of \texttt{ZPDVR}.
We first introduce the random stochastic gradient estimator in ZO methods and demonstrate the undiminished variance when employing it in composite optimization problems. 
Then, we delineate how the proposed method \texttt{ZPDVR} can reduce the two types of variances.

\subsection{Gradient Estimate in Zeroth-Order Optimization}
When only function evaluations are available, here, we utilize the Gaussian smoothing technique \citep{nesterov2017random} to derive the decent direction. 
Specifically, for the smoothing constant $v$ and the random vector $u\sim \NM(0, I_d)$, the directional derivative of $f_i$ in the direction $u$ for the smooth function $f_i$, $i\in [n]$, can be estimated as:
\begin{align}
    \hat{\nabla} f_i(x, u, v) = \frac{f_i(x + vu) - f_i(x )}{v} u.
    \label{eq:stochastic_directional_derivative}
\end{align}
Similarly, we define
\begin{align}
    \hat{\nabla} f(x, u, v) = \frac{1}{n}\sum_{i=1}^n\frac{f_i(x + v u) - f_i(x)}{v} u,
    \label{eq:directional_derivative}
\end{align}
as the approximation of the full directional gradient. 

Since the smoothing constant $v$ is fixed, for simplicity, we leave out $v$ in these gradient estimations and set 
\begin{align*}
    \hat{\nabla} f_i(x, u) := \hat{\nabla} f_i(x, u, v),
\end{align*}
and
\begin{align*}
    \hat{\nabla} f(x, u) := \hat{\nabla} f(x, u, v),
\end{align*}
in the rest of the paper. Under the smoothness condition, the following lemma and corollary show that these estimators defined in Eq.\eqref{eq:stochastic_directional_derivative} and Eq.\eqref{eq:directional_derivative} are nearly unbiased in expectation. The proof is deferred to Appendix \ref{sec:missing proof}.
\begin{lemma}
\label{lemma:nearly_approximation}
    Let the random vector $u$ drawn from the multivariate Gaussian distribution $\NM(0, I_d)$.
    For the $L$-smooth function $f_i$ and any $x\in \RB^d$, $i\in [n]$, the estimator in Eq.\eqref{eq:stochastic_directional_derivative} satisfies: 
    \begin{align}
        \hat{\nabla} f_i(x, u) = uu^\top \nabla f_i(x) + \frac{Lv}{2} s_i(x, u)\norm{u}u,\label{eq:near_fiux}
    \end{align}
    and its expectation w.r.t. $u$ is 
    \begin{align}
        \EB_u[\hat{\nabla} f_i(x, u)] = \nabla f_i(x) + \frac{Lv}{2}\tau_i(x, u), \label{eq:expect_fiux}
    \end{align}
    where $s_i(x, u)$ is a function of $x$ and $u$ within the range of $[0, 1]$, and $\tau_i(x, u)$ is the error term with $\Vert \tau_i \Vert  \leq (d+3)^{\frac{3}{2}}$.
    
    Furthermore, the expected norm of the estimation error between $\hat{\nabla} f_i(x, u)$ and $\nabla f(x)$ is bounded as 
    \begin{equation}
    \begin{aligned}
        &\quad\; \EB_{u}[\norm{\hat{\nabla} f_i(x, u) - \nabla f(x)}] \\
        &\leq 3(d+1)\norm{\nabla f(x)} + \frac{3L^2v^2}{4}(d+6)^3 + \\
        &\qquad 3(d+2) \norm{\nabla f_i(x) - \nabla f(x)}.
    \end{aligned}
    \label{eq:bound variance i}
    \end{equation}
\end{lemma}

In Eq.\eqref{eq:near_fiux}, $s_i(x, u)$ measures the curvature scaled by $L$ along the specified direction $u$ at the given point $x$. 
By taking the maximum value of $s_i(x, u)$ (i.e., $s_i(x, u) := 1$), we derive the upper bound of the distance between estimator $\hat{\nabla} f_i(x, u)$ and the full gradient $\nabla f(x)$ in Eq.\eqref{eq:bound variance i}.
This bound comprises three components: the norm of the true gradient $\nabla f(x)$; the trivial perturbation; and the last error $\norm{\nabla f_i(x) - \nabla f(x)}$ induced by the random sampling of $i$. 
Similarly, the following corollary establishes the nearly unbiased and bounded variance properties of $\hat{\nabla} f(x, u)$.

\begin{corollary}
\label{coro:near_fux}
    Let the random vector $u$ drawn from the multivariate Gaussian distribution $\NM(0, I_d)$.
    For the $L$-smooth function $f$ and any $x\in \RB^d$, $i\in [n]$, the full gradient estimator $\hat{\nabla} f(x, u)$ in Eq.\eqref{eq:directional_derivative} satisfies:
    \begin{align}
        \hat{\nabla} f(x, u) = uu^\top \nabla f(x) + \frac{Lv}{2} s(x, u)\norm{u}u\label{eq:near_fux},
    \end{align}
    with the expectation
    \begin{align}
        \EB_u[\hat{\nabla} f(x, u)] = \nabla f(x) + \frac{Lv}{2}\tau(x, u),\label{eq:expect_fux}
    \end{align}
    where $s(x, u) = \frac{1}{n}\sum_{i=1}^n s_i(x, u)$ and $\Vert \tau\Vert \leq (d+3)^{\frac{3}{2}}$.
    
    Moreover, the expected norm of the estimation error between $\hat{\nabla} f(x, u)$ and $\nabla f(x)$ is bounded as 
    \begin{equation}\label{eq:full directional variance bound}
    \begin{aligned}
            &\quad\; \EB_{u}[\norm{\hat{\nabla} f(x, u) - \nabla f(x)}] \\
            &\qquad \leq \frac{L^2v^2}{2}(d+6)^3 + 2(d+1) \norm{\nabla f(x)}.
        \end{aligned}
    \end{equation}
\end{corollary}

From Eq.\eqref{eq:expect_fiux} and Eq.\eqref{eq:expect_fux}, when the smoothing constant $v$ is sufficiently small, the following approximations hold: $\EB[\hn f_i(x, u)] \approx \nabla f_i(x)$ and $\EB[\hn f(x,u)] \approx \nabla f(x)$.
Besides, Eq.\eqref{eq:bound variance i} and Eq.\eqref{eq:full directional variance bound} indicate that both the expected norm of the estimation errors $\EB[\norm{\hat{\nabla} f_i(x, u) - \nabla f(x)}]$ and $\EB[\norm{\hat{\nabla} f(x, u) - \nabla f(x)}]$ depend on $\norm{\nabla f(x)}$. 
The presence of the norm of full gradient $\norm{\nabla f(x)}$ introduces the additional noise, which requires $\norm{\nabla f(x)}$ to gradually approach $0$ to obtain the precise estimation of $\nabla f(x)$ and guarantee convergence.

However, according to the FO optimality condition, for the objective function in Eq.\eqref{eq:objective_function}, the following relation holds at the optimal point $x^*$:
\begin{align*}
    0 \in \nabla f(x^*) + \partial \psi(x^*),
\end{align*}
which implies that $\norm{\nabla f(x^*)} \neq 0$ in the general case.
Therefore, as characterized by Eq.\eqref{eq:full directional variance bound} (Eq.\eqref{eq:bound variance i}), after excluding the trivial perturbation introduced by $v$, 
the variance of the estimator $\hat{\nabla} f(x, u)$ ($\hat{\nabla} f_i(x, u)$) persists as a non-zero value and does not diminish. 
Analogous to the impact of the sampling noise in SGD, with a constant learning rate $\eta$, this significantly compromises the convergence performance.

\begin{remark} 
In contrast to the random direction in Eq.\eqref{eq:near_fux}, researchers \citep{lian2016speedup, ji2019improved} utilize the finite-difference technique \citep{Kiefer1952StochasticEO} to approximate the gradient as follows, 
\begin{equation*}
 \sum_{l=1}^d \frac{f(x+v e_l) - f(x)}{v} e_l,
\end{equation*}
where $e_l\in \RB^d$ is the $l$-th standard basis vector.
Compared to the random gradient estimate in Eq.\eqref{eq:near_fux}, the above gradient estimate leverages all partial gradient information to approximate the gradient, and its variance is only influenced by the trivial perturbation arising from the smooth constant $v$.
In fact, it can be regarded as a first-order method and can find the correct solution when $\nabla f(x^*) \neq 0$.
However, it requires to take $\Omega(d)$ function evaluations, which is impractical in high-dimensional problems.
\end{remark}

\subsection{Coordinate-Wise Variance Reduction in Zeroth-Order Optimization}
Since $\norm{\nabla f(x^*)} \neq 0$, simply extending the SVRG methods \citep{johnson2013accelerating, xiao2014proximal} to composite as well as constrained optimization problems can not eliminate the overall variance - such methods only reduce the variance induced by the stochastic sampling $\norm{\nabla f_i(x) - \nabla f(x)}$ in Eq.\eqref{eq:bound variance i} and omit the coordinate-wise variance. 
Although driving all partial derivative information can mitigate the coordinate-wise variance \citep{huang2019faster, kazemi2023efficient}, the computation is prohibitive for high-dimensional problems.
To reduce the overall variance while preserving computational efficiency, we propose a novel gradient estimator, utilizing only the random gradient estimate throughout the iterative process.

According to Eq.\eqref{eq:near_fux}, when $v$ is sufficient small, we have that 
\begin{align}
\label{eq:apd}
    \hat{\nabla} f(x, u) \approx u^\top \nabla f(x) u,
\end{align}
which is an approximation of the directional derivative in the direction $u$. Here, we leverage the directional derivative estimation and build a gradient estimator $\td{h}$ by the averaging trick, i.e., 
\begin{align}
    \td{h}_{k+1} = \td{h}_k + \frac{1}{d+2}(\hat{\nabla} f(x_k, u) - uu^\top \td{h}_k) \label{eq:tdh_definition}.
\end{align}
The estimator $\td{h}$ employs the difference between the estimated directional derivative $\hat{\nabla} f(x_k, u)$ and the value of $h$ in the same direction $u$ to refine the gradient estimation error.
Moreover, the coefficient $\frac{1}{d+2}$ further smoothens this correction across all dimensions.
In \cref{sec:convergence}, we will demonstrate that by incrementally refining the gradient estimated error at disparate directions, as $x \to x^*$, the estimator $\td{h}$ can gradually converge to $\nabla f(x^*)$.

\begin{remark}
We can also connect the estimator $\td{h}$ with the sketch-and-project method \citep{hanzely2018sega}.
Since $u^\top \nabla f(x)$ in Eq.\eqref{eq:apd} is the gradient sketch, according to \citet{hanzely2018sega}, $\td{h}_{k+1}$ can be obtained by solving the following optimization problem,
\begin{equation*}
    \begin{aligned}
    \min_{\td{h}} \quad & \norm{\td{h} - \td{h}_k}, \\
    \text{s.t.} \quad & u^\top h = u^\top \nabla f(x_k).
    \end{aligned}
\end{equation*}
By solving the Lagrangian function $\norm{h-h_k} + \lambda u^\top(\td{h}-\nabla f(x_k))$, we have the closed-form solution:
\begin{align*}
    h^{k+1} = h_k  + \frac{uu^\top (\nabla f(x_k) - \td{h}_k)}{\norm{u}}.
\end{align*}
Compared to the above update formula, Eq.\eqref{eq:tdh_definition} simplifies the computation of the norm $\norm{u}$ by approximating it with its expectation ($\EB[\norm{u}] = d$, see \cref{lemma:proj_gaussian_lemma} in \cref{sec:technique lemma}) plus a constant.
Therefore, from the view of the sketch-and-projection method, Eq.\eqref{eq:tdh_definition} essentially requires that the sketches of the estimator $\td{h}$ and gradient $\nabla f(x)$ are identical. 
Hence, when the projections of the gradient $\nabla f(x)$ and the estimator $\td{h}$ align along numerous directions, $\td{h}$ gradually becomes identical to the gradient.
\end{remark}

In Algorithm \ref{algo:proposed_algorithm}, we present the proposed algorithm \texttt{ZPDVR}, utilizing the proposed estimator $\td{h}$ and SVRG techniques to reduce both the sampling variance and coordinate-wise variance in the composite optimization problem Eq.\eqref{eq:objective_function}.
In \texttt{ZPDVR}, like SVRG, we maintain a reference point $w_k$ and pass over the full data to compute the gradient estimate $\td{\nabla} f(w_k)$ at the reference point,
\begin{align}
\label{eq:tdfw_definition}
    \td{\nabla} f(w_k) = \td{h}_k + \hat{\nabla} f(w_k, u) - uu^\top \td{h}_k,
\end{align}
where we leverage the directional gradient $\hat{\nabla} f(w_k, u)$ to refine the gradient estimator $\td{h}_k$ in the direction $u$.

Furthermore, with the conserved $\td{\nabla} f(w_k)$, we can derive the stochastic gradient estimate $g_k$,
\begin{align}
    g_k = \hat{\nabla} f_i(x_k, u_k) - \hat{\nabla} f_i(w_k, u_k) + \td{\nabla} f(w_k),
    \label{eq:g_definition}
\end{align}
which reduces the sampling variance. 
Then, the update of $x_{k+1}$ is 
\begin{align}
\label{eq:update_x}
    x_{k+1} = \prox(x - \eta g_k).
\end{align}

Notably, with Eq.\eqref{eq:tdfw_definition} and Eq.\eqref{eq:g_definition}, we demonstrate that $g_k$ is also a nearly unbiased estimation of $\nabla f(x_k)$ in the following Lemma. The proof is deferred to Appendix \ref{sec:missing proof}.
\begin{lemma}
\label{lemma:nearly_gk}
For the $L$-smooth function $f_i$, $i\in [n]$, the expected value of $g_k$ defined in Eq.\eqref{eq:g_definition} is 
    \begin{align}
    \label{eq:expectation_of_gk}
        \EB_{u,u_k,i}[g_k] = \nabla f(x_k) + \frac{Lv}{2}\tau_{i, k},
    \end{align}
    where $\tau_{i, k}= \EB_{u, u_k, i}[(s_i(x_k, u_k) - s_i(w_k, u_k))\norm{u_k}u_k + s(w_k, u)\norm{u} u]$ with the norm $\Vert \tau_{i, k}\Vert \leq 2(d+3)^{\frac{3}{2}}$.
\end{lemma}

\begin{algorithm}[tb]
\caption{Zeroth-order Proximal Doubly Variance Reduction}\label{algo:proposed_algorithm}
    \begin{algorithmic}[1]
\renewcommand{\algorithmicrequire}{\textbf{Input:}}
\renewcommand{\algorithmicensure}{\textbf{Output:}}
  \REQUIRE $x_0$, $w_0 = x_0$, $\td{h}_0$, the probability $p$, and the learning rate $\eta$.
  \FOR{$k=0, \dots, K-1$}
    \IF{$k = 0 \textbf{ or } w_{k} \neq w_{k-1}$}
      \STATE Sample a new vector $u \sim \NM(0, I)$ and save it
      \STATE $\td{\nabla} f(w_k) = \td{h}_k + \hat{\nabla} f(w_k, u) - uu^\top \td{h}_k$
    \ELSE
      \STATE $\td{\nabla} f(w_k) = \td{\nabla} f(w_{k-1})$
    \ENDIF
    \STATE Sample $u_k\sim \NM(0, I)$ and pick $i \in \{1, \dots, n\}$ uniformly at random
    \STATE $g_k = \hat{\nabla} f_i(x_k, u_k) - \hat{\nabla} f_i(w_k, u_k) + \td{\nabla} f(w_k)$
    \STATE $x_{k+1} = \prox(x_k - \eta g_k)$
    \STATE Draw $z_k$ from the uniform distribution $U[0, 1]$
    \IF{$z_k \leq p$}
      \STATE $w_{k+1} = x_k$
      \STATE $\td{h}_{k+1} = \td{h}_k + \frac{1}{d+2}( \hat{\nabla} f(x_k, u) - uu^\top\td{h}_k)$
    \ELSE
      \STATE $w_{k+1} = w_k$
      \STATE $\td{h}_{k+1} = \td{h}_k$
    \ENDIF
  \ENDFOR
  \ENSURE  $x_K$
\end{algorithmic}
\end{algorithm}

In Algorithm \ref{algo:proposed_algorithm}, instead of the classic double loop structure, we adopt a loopless scheme \citep{kovalev2020don, qian2021svrg}: with a small probability $p$, we update the reference point $w_{k+1}$ with the value of $x_k$ and take a full pass over the $n$ data points to refine the estimate $\td{h}_{k+1}$ in the direction $u$; with the probability $1-p$, we remain the previous $w_k$ and $\td{h}_k$ and utilize them in the next iteration. 

Furthermore, with the probability $p$ to perform a full pass over the dataset, \texttt{ZPDVR} has to query the function value $\OM(pn)$ times in the expectation at each iteration. When $p=\frac{1}{n}$, the proposed method \texttt{ZPDVR} recovers the complexity of the FO-SVRG methods. Compared to the complexity $\mathcal{O}(d)$ in other ZO variance reduction methods \citep{huang2019faster, kazemi2023efficient}, \texttt{ZPDVR} is independence of the dimension size $d$, which makes it efficient and practical for high-dimensional problems. Besides, in \texttt{ZPDVR}, directions $u$ and $u_k$ drawn in Eq.\eqref{eq:tdfw_definition} and Eq.\eqref{eq:g_definition} are independent. Hence, in Algorithm \ref{algo:proposed_algorithm}, except for $\td{\nabla} f(w_k)$, we also have to save $u$ to update $\td{h}_{k+1}$ when $z_k \leq p$.

\begin{remark}
Note that, in \cref{algo:proposed_algorithm}, the update of the reference point $w$ and the gradient estimate $\td{h}$ are bound together. 
Since $w$ is updated with the probability $p$, the estimator $\td{h}$ is not required to be updated at every iteration, which differs significantly from \citep{hanzely2018sega}.
We demonstrate that sporadically updating the estimator $\td{h}$ can reduce the coordinate-wise variance, and \texttt{ZPDVR} is not a trivial extension.
Moreover, according to the update of $w$ in \cref{algo:proposed_algorithm} and Eq.\eqref{eq:tdh_definition}, $\td{h}_k$ is actually estimating the gradient at the reference point, $\nabla f(w_k)$. 
Then, when $w \to x^*$, we can obtain that $\td{h} \to \nabla f(x^*)$.    
\end{remark}

\section{Convergence Analysis}
\label{sec:convergence}
This section will establish the convergence property of the proposed method \texttt{ZPVR}. 
Initially, we define the following Lyapunov function $\Psi$:
\begin{equation}
    \begin{aligned}
    \Psi(x_k) &= \norm{x_k - x^*} + \alpha \norm{\td{h}_k - \nabla f(x^*)} + \\
    &\qquad \frac{\beta}{n}\sum_{i=1}^n \norm{\nabla f_i(w_{k}) - \nabla f_i(x^*)},
    \label{eq:Lyapunov_function}
\end{aligned}
\end{equation}
where both $\alpha$ and $\beta$ are positive real values.


The Lyapunov function in Eq.\eqref{eq:Lyapunov_function} comprises three distinct terms: the first term quantifies the distance between the current point $x_k$ and the optimal point $x^*$; the second term, $\norm{\td{h}_k - \nabla f(x^*)}$, delineates the disparity between the gradient estimator $\td{h}_k$ and the gradient at $x^*$; and the final term serves as an approximation for the difference between the reference point $w_k$ and $x^*$ in the strongly convex setting.
During the training process, if the value of the Lyapunov function \eqref{eq:Lyapunov_function} converges to $0$, we have that both $x_k$ and $w_k$ attain the optimal point $x^*$ and the gradient estimator $\td{h}_k$ approaches $\nabla f(x^*)$.
According to the definition of the Laypunov function, it involves successive iterations with respect to $\norm{x_k - x^*}$, $\norm{\td{h}_k - \nabla f(x^*)}$, and $\frac{1}{n}\sum_{i=1}^n \norm{\nabla f_i(w_{k}) - \nabla f_i(x^*)}$. 
Next, we shall establish recursive formulations for these components in the subsequent three lemmas, respectively. The proofs are provided in Appendix \ref{sec:missing proof}.

\begin{lemma}
\label{lemma:recursion_of_1}
    For the $\mu$-strongly and $L$ smooth function $f_i$, $i\in[n]$, and the convex function $\psi(x)$, according to \cref{algo:proposed_algorithm}, we have
    \begin{equation}\label{eq:recursion_of_1}
        \begin{aligned}
            &\quad\; \EB[\norm{x_{k+1} - x^*}] \\
            &\leq (1-\frac{\mu}{2}\eta)\norm{x_k - x^*} - 2\eta\Big(f(x_k) - f(x^*) - \\
            &\qquad \dotprod{\nabla f(x^*), x_k - x^*}\Big) + \frac{2(d+3)^3L^2v^2\eta}{\mu} + \\
            &\qquad \eta^2 \EB[\norm{g_k - \nabla f(x^*)}].
        \end{aligned}
    \end{equation}
\end{lemma}

In \cref{lemma:recursion_of_1}, the convexity of $f$ indicates the positive value of $f(x_k) - f(x^*) - \dotprod{\nabla f(x^*), x_k - x^*}$. 
Except for the ZO perturbation, the iteration of $x$ involves the estimation gap between the stochastic gradient estimation $g_k$ and $\nabla f(x^*)$, i.e., $\EB[\norm{g_k - \nabla f(x^*)}]$, which will be bounded in the later analysis.

\begin{lemma}
\label{lemma:recursion_of_2}
    For the $L$ smooth function $f_i$, $i\in [n]$, following Algorithm \ref{algo:proposed_algorithm}, we have 
    \begin{equation}
    \label{eq:recursion_of_2}
        \begin{aligned}
            &\quad\; \EB[\norm{\td{h}_{k+1} - \nabla f(x^*)}] \\
            &\leq (1-\frac{p}{2(d+2)})\norm{\td{h}_k - \nabla f(x^*)} + \frac{5(d+6)^3L^2v^2p}{2(d+2)} + \\
            &\qquad \frac{8Lp}{3(d+2)}(f(x_k) - f(x^*) - \dotprod{\nabla f(x^*), x_k - x^*}).
        \end{aligned}
    \end{equation}
\end{lemma}


\begin{lemma}
For the $L$ smooth function $f_i$, $i\in [n]$, the recursive formulation of $\frac{1}{n}\sum_{i=1}^n \norm{\nabla f_i(w_{k}) - \nabla f_i(x^*)}$ is as follows:
    \label{lemma:recursion_of_3}
    \begin{equation}
    \label{eq:recursion_of_3}
        \begin{aligned}
                &\quad\; \EB[\frac{1}{n}\sum_{i=1}^n \norm{\nabla f_i(w_{k+1}) - \nabla f_i(x^*)}] \\
                &\leq \frac{(1-p)}{n}\sum_{i=1}^n \norm{\nabla f_i(w_k) - \nabla f_i(x^*)} + \\
                &\qquad 2Lp(f(x_k) - f(x^*) - \dotprod{\nabla f(x^*), x_k -x^*}).
        \end{aligned}
    \end{equation}
\end{lemma}

According to \cref{lemma:recursion_of_2} and \cref{lemma:recursion_of_3}, $\EB[\norm{\td{h}_k - \nabla f(x^*)}]$ and $\EB[\frac{1}{n}\sum_{i=1}^n \norm{\nabla f_i(w_k) - \nabla f_i(x^*)}]$ is reduced by $\OM(1-\frac{p}{d})$ and $\OM(1-p)$, respectively.
Furthermore, all recursive formulations of these three components involve the term $f(x_k) - f(x^*) - \dotprod{\nabla f(x^*), x_k -x^*}$.
In the subsequent lemma, we will establish a connection between the term $\norm{g_k - \nabla f(x^*)}$ introduced in Lemma \ref{lemma:recursion_of_1} and $\norm{\td{h}_k - \nabla f(x^*)}$, driving the bound for the former item.

\begin{lemma}
    \label{lemma:gf_relationship}
    For the $L$ smooth function $f_i$, $i\in[m]$, the following inequality holds
    \begin{equation}
    \label{eq:gf_relationship}
        \begin{aligned}
            &\quad\; \EB[\norm{g_k - \nabla f(x^*)}]\\
            &\leq 2(2d+3)\norm{\td{h}_k - \nabla f(x^*)} + 4L^2v^2(d+6)^3 + \\
            &\qquad \frac{4(2d+3)}{n}\sum_{i=1}^n \norm{\nabla f_i(w_k) -\nabla f_i(x^*)} + \\
            &\qquad 4(d+3)L(f(x_k) - f(x^*) - \dotprod{\nabla f(x^*), x_k -x^*}) .
        \end{aligned}
    \end{equation}
\end{lemma}

From \cref{lemma:gf_relationship}, the bound of $\EB[\norm{g_k - \nabla f(x^*)}]$ involves $\td{h}_k - \nabla f(x^*)$, $f(x_k) - f(x^*) - \dotprod{\nabla f(x^*), x_k -x^*}$, and $\frac{1}{n}\sum_{i=1}^n \norm{\nabla f_i(w_k) -\nabla f_i(x^*)}$. 
Hence, setting $p = \frac{1}{n}$ and combing with \cref{lemma:recursion_of_1}, \cref{lemma:recursion_of_2}, and \cref{lemma:recursion_of_3}, we can obtain the whole recursive formula of the Lyapunov function in the following corollary:

\begin{corollary}
\label{coro:psi_recursion}
    For the $L$ smooth and $\mu$-strongly convex function $f_i$, $i\in [n]$, and the convex function $\psi(x)$, when $p = \frac{1}{n}$, the recursive formula of the Lyapunov function $\Psi$ is following: 
    \begin{equation}
    \label{eq:psi_recursion}
        \begin{aligned}
            &\quad\; \EB[\Psi(x_{k+1})]\\
    &\leq (1-\frac{\mu}{2}\eta)\norm{x_k - x^*} + \alpha'\norm{\td{h}_k - \nabla f(x^*)} + \\
        &\qquad \beta'\frac{1}{n}\sum_{i=1}^n \norm{\nabla f_i(w_k) - \nabla f_i(x^*)} - \gamma C + \Delta,
        \end{aligned}
    \end{equation}
where $ \alpha' = \alpha(1-\frac{1}{2n(d+2)} + \frac{2(2d+3)\eta^2}{\alpha})$, $\beta' = \beta(1-\frac{1}{n} + \frac{4(2d+3)\eta^2}{\beta})$, $\gamma = 2\eta - 4(d+3)L\eta^2 - \frac{8L}{3n(d+2)}\alpha - \frac{2L}{n}\beta$, $C = f(x_k) - f(x^*) - \dotprod{\nabla f(x^*), x_k - x^*}$, and $\Delta =  \frac{2(d+3)^3L^2v^2\eta}{\mu} + 4L^2v^2(d+6)^3\eta^2 + \frac{5(d+6)^3L^2v^2}{2n(d+2)}\alpha$.
\end{corollary}

According to \cref{coro:psi_recursion}, by carefully selection $\eta$, $\alpha$, and $\beta$, we can guarantee that $\gamma \leq 0$ and $0< (1-\frac{\mu\eta}{2}),  \alpha', \beta'< 1$.
Hence, we can drive the convergence rate of the Lyapunov function, which is presented in the following theorem.

\begin{theorem}
\label{thm:psi_convergence}
Let $p=\frac{1}{n}$, $\alpha = 8n(d+2)(2d+3)\eta^2$, $\beta = 8n(2d+3)\eta^2$, and $\eta = \frac{1}{(40d+63)L}$. Given the $L$ smooth and $u$-strongly convex function $f_i$, $i\in [n]$, and the convex function $\psi(x)$, for $k \in {0, \dots, K-1}$, the Lyapunov function satisfies 
    \begin{equation}
    \begin{aligned}
    \label{eq:psi_convergence}
        &\quad\; \EB[\Psi(x_{k+1})] \\
        &\leq \max\{1-\frac{1}{\kappa(80d+126)}, 1-\frac{1}{4n(d+2)}\}\Psi(x_k) + \delta,
    \end{aligned}
    \end{equation}
    where $\delta = \frac{2(d+3)^3v^2\kappa}{40d+63} + \frac{8(5d+8)(d+6)^3v^2}{(40d+63)^2}$.
\end{theorem}

\begin{table}[t]
\caption{Summary of data sets and regularization coefficients ($\lambda_1$ and $\lambda_2$) used in our experiments.}
\label{tab:dataset}
\vskip 0.1in
\begin{center}
\begin{tabular}{lccccc}
\toprule
Data Set & n      & d    & $\lambda_1$ & $\lambda_2$ & $v$\\ \midrule
a9a      & 32561  & 123  &      $10^{-4}$       &       $10^{-4}$       & $10^{-3}$\\
w8a      & 49749  & 300  &      $10^{-5}$       &       $10^{-4}$       & $10^{-3}$\\
covtype  & 581012 & 54   &      $10^{-4}$       &       $10^{-5}$       & $10^{-3}$\\
gisette  & 6000   & 5000 &      $10^{-4}$       &       $10^{-4}$       & $10^{-3}$\\ \bottomrule
\end{tabular}
\end{center}
\vskip -0.1in
\end{table}

According to \cref{thm:psi_convergence}, we can telescope the inequality \eqref{eq:psi_convergence} from $k=0$ for $K-1$, and then find the SZO query complexity and the value of the smooth constant $v$ to get the $\epsilon$ accuracy solution.

\begin{corollary}
    \label{coro:convergence_property}
    Let $p=\frac{1}{n}$, $\alpha = 8n(d+2)(2d+3)\eta^2$, $\beta = 8n(2d+3)\eta^2$, and $\eta = \frac{1}{(40d+63)L}$. Given the $L$ smooth and $u$-strongly convex function $f_i$, $i\in [n]$, and the convex function $\psi(x)$, we have
    \begin{align}
    \label{eq:convergence_property}
        \EB[\Psi(x_K)] \leq (1-\theta)^K \Psi(x_0) + \sigma,
    \end{align}
    where $\theta = \frac{1}{\kappa(80d+126) + 4n(d+2)}$, and $\sigma = [(80d+126) + 4n(d+2)]\kappa(\kappa + 1)(d+6)^2v^2$.
\end{corollary}

From \cref{coro:convergence_property}, in order to get $\epsilon$ accuracy for the Lyapunov function as well as $F$, the smooth constant $v$ need to sufficiently small such that $v =\OM({\sqrt{\frac{\epsilon}{\kappa^2 d^3 n}}})$. We can also derive the $\OM(d(n+\kappa)\log(\frac{1}{\epsilon}))$ SZO query complexity of \texttt{ZPDVR}, which aligns with the best-known results of SVRG methods.
The optimal SZO query complexity demonstrates the efficiency of \texttt{ZPDVR}.

\section{Experiment}
\label{sec:experiment}

\begin{figure*}[t]
\centering
\vskip 0.15in
\subfigure[a9a]{
    \includegraphics[width=0.8\columnwidth]{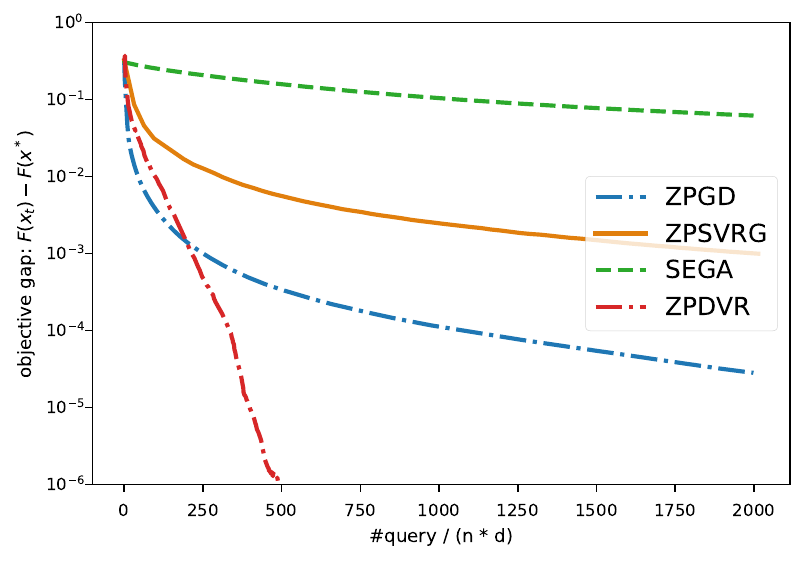}
    \label{fig:a9a}
}
\subfigure[w8a]{
    \includegraphics[width=0.8\columnwidth]{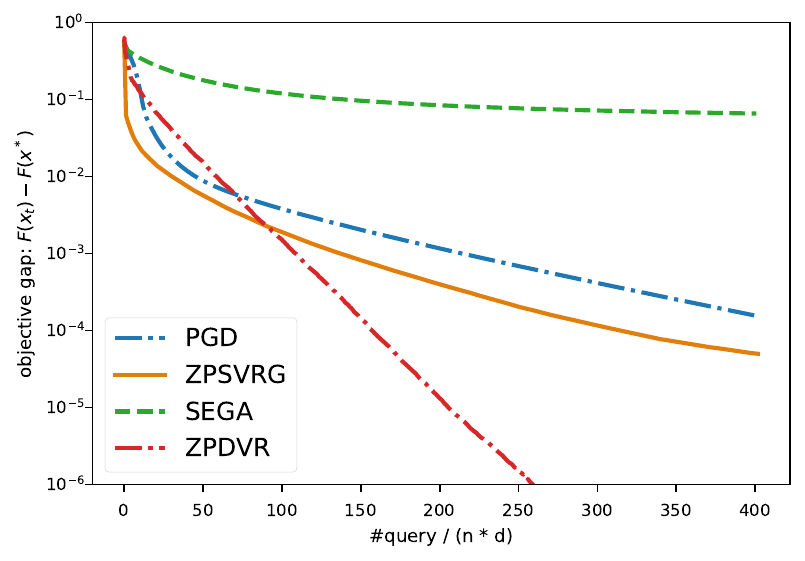}
    \label{fig:w8a}
}
\subfigure[covtype]{
    \includegraphics[width=0.8\columnwidth]{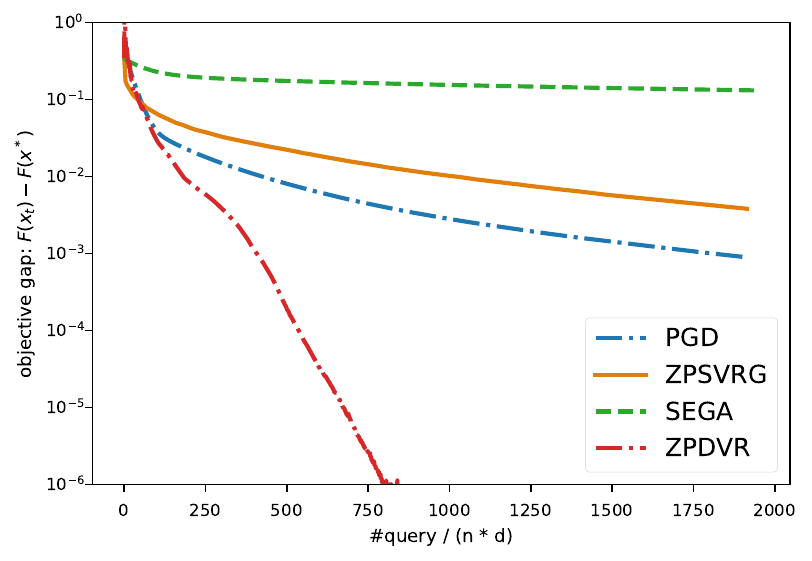}
    \label{fig:covtype}
}
\subfigure[gisette]{
    \includegraphics[width=0.8\columnwidth]{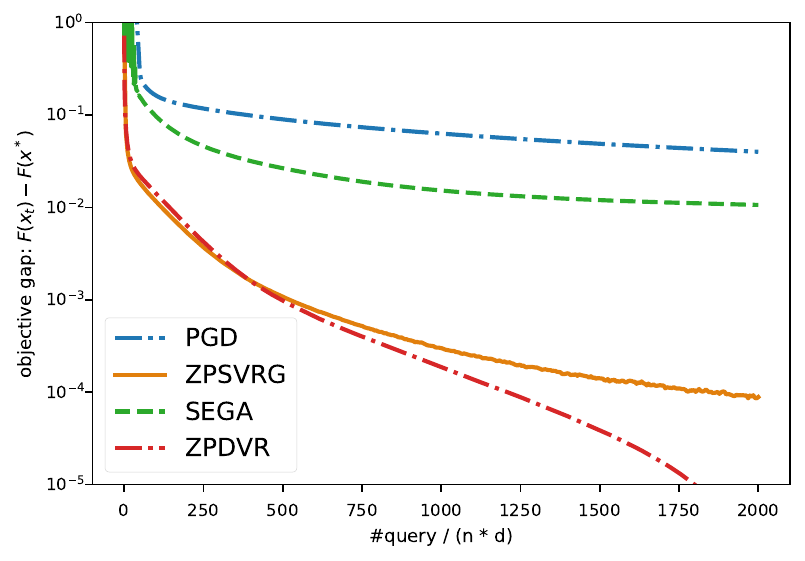}
    \label{fig:gisette}
}
\caption{Comparison of different zeroth-order methods for the loss residual $F(x) - F(x^*)$ versus the number of SZO. The $y$ axis is on a logarithmic scale and the $x$ label is the number of SZO divided by $n*d$.}
\label{fig:experiment_result}
\vskip -0.1in
\end{figure*}


In this section, we conduct several numerical experiments to demonstrate the convergence propriety of the proposed \texttt{ZPDVR} and compare its performance against other related FO and ZO methods.
Here, we focus on a binary logistic classification problem. 
Given the data set $\DM=\{(z_i, y_i)\}_{i=1}^n$, where $z_i\in \RB^d$ is the feature vector and $y_i\in \{-1, 1\}$ is the label, the strongly convex objective function $F(x)$ is the regularized cross entropy loss, i.e.,
\begin{align*}
    \min_{x\in \dom} \frac{1}{n}\sum_{i=1}^n \log(1+\exp(-y_ix^\top z_i) + \frac{\lambda_2}{2}\norm{x} + \lambda_1\Vert x\Vert_1,
\end{align*}
with  $f_i(x) = \log(1+\exp(-y_i x^\top z_i)) + \frac{\lambda_2}{2}\norm{x}$ and $\psi(x) = \lambda \Vert x\Vert_1$, for $i\in [n]$.

Specifically, we choose three related FO and ZO methods as baselines: Proximal Gradient Descent (\texttt{PGD}), \texttt{ZPSVRG}, and \texttt{SEGA}. 
In our experiment, \texttt{ZPSVRG} carries out the ZO version of proximal variance reduction method \citep{xiao2014proximal} with the random gradient estimate defined in Eq.\eqref{eq:directional_derivative}, which has the non-vanishing coordinate-wise variance.


In \cref{tab:dataset}, we present the regularization coefficients and the smooth constant used in four binary classification data sets from LIBSVM website\footnote{\url{https://www.csie.ntu.edu.tw/~cjlin/libsvmtools/datasets/binary.html}}.
To enhance implementation efficiency, \texttt{ZPSVRG}, \texttt{ZPDVR}, and \texttt{SEGA} are executed in the batch version. 
The grid search is performed to identify optimal hyperparameters for each method.
The detailed information of the hyperparameter tuning procedure is provided in \cref{sec:hyperparameter}. 

As shown in \cref{fig:experiment_result}, we compare the objective gap $F(x) - F(x^*)$ of \texttt{ZPDVR} with other methods in terms of SZO query complexity. 
Across all data sets, \texttt{ZPDVR} achieves the best performance and presents the linear convergence rate as demonstrated in \cref{thm:psi_convergence}. 
Besides, \texttt{PGD} and \texttt{SEGA} also obtain the linear convergence rate, albeit at a slower pace than \texttt{ZPDVR}. 
While due to the non-vanishing coordinate-wise variance in \texttt{ZPSVRG}, we observe that the function value of \texttt{ZPSVRG} can not decrease to the specified precision.
For instance, in \cref{fig:gisette}, \texttt{ZPSVRG} stagnates in the neighborhood of $F(x^*)$ with the radius $10^{-4}$.
These results emphasize the linear convergence rate and double variance reduction properties of \texttt{ZPDVR} analyzed in \cref{sec:convergence} for the composite problem \eqref{eq:objective_function}.

\section{Conclusion}
\label{sec:conclusion}

This paper introduces a novel ZO variance reduction method, \texttt{ZPDVR}, designed to diminish both the inherent sampling variance and the coordinate-wise variance in composite optimization problems. 
Unlike prior method, \texttt{ZPDVR} obviates the need for $\Omega(d)$ function evaluations to approximate FO information. 
It instead only requires $\OM(1)$ SZO query in expectation per iteration, which enhances the computational efficiency, particularly for high-dimensional problems. 
Additionally, we establish the linear convergence of \texttt{ZPDVR} and show that its SZO query complexity is $\OM(d(n+\kappa)\log(\frac{1}{\epsilon}))$, paralleling the optimal results achieved by SVRG methods. 
Moreover, \texttt{ZPDVR} is adaptable to constrained optimization scenarios, including black-box optimization problems. 
Empirical results validate both the convergence properties and the superior performance of \texttt{ZPDVR}.

\section*{Acknowledge}
This work was supported by the National Natural Science Foundation of China under Grant 12101491,
the National Natural Science Foundation for Outstanding Young Scholars of China under Grant 72122018,
the MOE Project of Key Research Institute of Humanities and Social Sciences No. 22JJD110001,
and A*star Centre for Frontier AI Research.

\section*{Impact Statement}

This paper presents work whose goal is to advance the field of Machine Learning. 
There are many potential societal consequences of our work, none which we feel must be specifically highlighted here.


\bibliography{reference}
\bibliographystyle{icml2024}

\newpage
\appendix
\onecolumn

\section{Some useful Lemmas}
\label{sec:technique lemma}
\begin{lemma}[Cauchy–Schwarz inequality]
\label{lemma:proj_sum_lemma}
For all vectors $x_1, \dots, x_n \in \RB^d$, we have
    \begin{align}
    \label{eq:proj_sum_lemma}
        \norm{\sum_{i=1}^n x_i} \leq n\sum_{i=1}^n \norm{x_i}.
    \end{align}
\end{lemma}

\begin{lemma}[Young's inequality]
    \label{lemma:proj_Young}
    For any two vectors $x, y \in \RB^d$, we have
    \begin{align}
        \dotprod{x, y} \leq \frac{a\norm{x}}{2} + \frac{\norm{y}}{2a},\label{eq:proj_young}
    \end{align}
    where $a$ is a positive real number.
\end{lemma}

\begin{lemma}[Theorem 2.1.5 of \citet{nesterov2018lectures}]
    \label{lemma:proj_smooth}
    For a $L$-smooth function $f$ defined in the domain $\XM$, $\forall x, y \in \XM$, we have
    \begin{align}
        f(x) \geq f(y) + \dotprod{\nabla f(y), x - y} + \frac{1}{2L}\norm{\nabla f(x) - \nabla f(y)}.
        \label{eq:proj_smooth}
    \end{align}
\end{lemma}

\begin{lemma}
    \label{lemma:proj_gaussian_lemma}
    Let $A$ and $B$ be two symmetric matrices. Random vector $u$ has the Gaussian distribution, i.e., $u\sim \NM(0, I_d)$. Then, we have
    \begin{align}
    \label{eq:proj_gaussian_lemma}
        \EB_u[u^\top A u\cdot u^\top B u] = (\tr A) (\tr B) + 2 (\tr AB).
    \end{align}
\end{lemma}

\begin{corollary}
    \label{coro:proj_gaussian_lemma2}
    Let $v$ be a any vector in $\RB^d$. For the random vector $u$ with the Gaussian distribution, i.e., $u\sim \NM(0, I_d)$, we have
    \begin{align}
    \label{eq:proj_gaussian_lemma2}
        \EB_u[\norm{u u^\top v}] = (d+2)\norm{v}.
    \end{align}
\end{corollary}
\begin{proof}
\begin{align*}
    \EB_u [\norm{u u^\top v}] &= \EB_u [\tr(v^\top u u^\top uu^\top v)]\\
    &= \EB_u [\tr(u^\top u u^\top vv^\top u)]\\
    &= \EB_u [u^\top I u \cdot u^\top vv^\top u]\\
    &\leftstackrel{\eqref{eq:proj_gaussian_lemma}}{=} \tr(I)\tr( vv^\top) + 2 \tr (vv^\top)\\
    &= (d+2)\tr(vv^\top)\\
    &= (d+2)\norm{v}.
\end{align*}
\end{proof}
\begin{lemma}[Lemma 1 of \citet{nesterov2017random}]
\label{lemma:bound_random_norm}
Let the random vector $u\sim \mathcal{N}(0, I_{d})$. we have
\begin{align}
\label{eq:bound_random_norm}
    \begin{matrix}
        &\EB[\Vert u \Vert^q] \leq d^{\frac{q}{2}} \quad & q\in[0, 2],\\
        &d^{\frac{q}{2}}\leq \EB[\Vert u \Vert^q]\leq (q+d)^{\frac{q}{2}} \quad & q\geq 2.
    \end{matrix}
\end{align}
    
\end{lemma}

\section{Missing Proofs}
\label{sec:missing proof}

\begin{proof}[\textbf{Proof of Lemma \ref{lemma:nearly_approximation}.}]
    For the $L$-smooth function $f_i$, we have the following Taylor expansion,
    \begin{equation}
        \begin{aligned}
        f_i(x+vu) &= f_i(x) + v\dotprod{\nabla f_i(x), u} + \frac{v^2}{2}u^\top \nabla^2 f_i(x') u,
    \end{aligned}
    \label{eq:appendix_f_i}
    \end{equation}
    where $x'\in (x, x + vu)$.
    Plugging it into Eq.\eqref{eq:stochastic_directional_derivative}, we have
    \begin{equation}
        \label{eq:appendix_expectation_f_i}
        \begin{aligned}
        \hat{\nabla} f_i(x, u) &= u\dotprod{u, \nabla f_i(x)} + \frac{v}{2}u^\top \nabla^2 f_i(x')u u\\
        &=uu^\top \nabla f_i(x) + \frac{Lv}{2}s_{i}(x, u)\norm{u}u,
    \end{aligned}
    \end{equation}
    where the last equality employs the fact that $0 \preceq \nabla^2 f_i(x') \preceq L$ for any accessible $x'$, and the function $s_i(x,u )$ is confined to the range $[0, 1]$.

    Taking the expectation w.r.t. $u$ for $\hat{\nabla} f_i(x)$, we have
    \begin{align*}
        \EB[\hat{\nabla } f_i(x, u)] &= \EB[uu^\top]\nabla f_i(x)+ \frac{Lv}{2}\EB[s_i(x, u)\norm{u}u]\\
        &= \nabla f_i(x) + \frac{Lv}{2}\EB[s_i(x, u)\norm{u}u].
    \end{align*}

    Since $\Big\Vert \EB[s_i(x, u)\norm{u}u]\Big\Vert \leq \EB\Big[\Big\Vert s_i(x, u) \norm{u} u\Big\Vert\Big] \leq \EB\Big[\Big\Vert \norm{u} u\Big\Vert\Big] = \EB[\Vert u\Vert^{3}]$, with Eq.\eqref{eq:bound_random_norm} $\EB[\Vert u\Vert^3] \leq (d+3)^{\frac{3}{2}}$, we then have $\Big\Vert \EB[s_i(x, u)\norm{u}u]\Big\Vert \leq (d+3)^{\frac{3}{2}}$. 

    For the expected norm, we have
    \begin{align*}
        &\quad\; \EB[\norm{\hat{\nabla} f_i(x, u) - \nabla f(x)}] \\
        &\stackrel{\eqref{eq:near_fiux}}{=} \EB[\norm{uu^\top \nabla f_i(x)+ \frac{Lv}{2}s_i(x, u)\norm{u}u - \nabla f(x)}] \\
        &= \EB\Big[\Big\Vert uu^\top \nabla f_i(x) - uu^\top \nabla f(x) + uu^\top \nabla f(x) - \nabla f(x)+ \frac{Lv}{2}s_i(x, u)\norm{u}u\Big\Vert^2\Big]\\
        &\stackrel{\eqref{eq:proj_sum_lemma}}{\leq} 3\EB[\norm{uu^\top(\nabla f_i(x) - \nabla f(x)}] + 3\EB[\norm{(uu^\top - I)\nabla f(x)}] + \frac{3L^2v^2}{4}\EB[\Vert u\Vert^6] \\
        &\stackrel{\eqref{eq:proj_gaussian_lemma2}}{\leq} 3(d+2)\norm{\nabla f_i(x) - \nabla f(x)} + 3(d+1) \norm{\nabla f(x)} + \frac{3L^2v^2}{4}\EB[\Vert u\Vert^6].
    \end{align*}
    Combining with \cref{lemma:bound_random_norm}, we conclude the result in \cref{lemma:nearly_approximation}.
\end{proof}

\begin{proof}[\textbf{Proof of \cref{coro:near_fux}.}]
According to \cref{lemma:nearly_approximation} and the definition of $\hat{\nabla} f(x, u)$, we have that 
\begin{align*}
    \hat{\nabla} f(x, u) = uu^\top \nabla f(x) + \frac{Lv}{2}s(x, u)\norm{u}u,
\end{align*}
where $s(x, u) = \frac{1}{n}\sum_{i=1}^n s_i(x, u)$.

Hence, taking the expectation w.r.t. $u$, we also derive 
    \begin{align*}
        \EB[\hat{\nabla } f(x, u)] &= \EB[uu^\top]\nabla f(x)+ \frac{Lv}{2}s(x, u)\EB[\norm{u}u]\\
        &= \nabla f_i(x) + \frac{Lv}{2}\EB[s(x, u)\norm{u}u],
    \end{align*}
where $\Big\Vert \EB[s(x, u)\norm{u}u] \Big\Vert \leq (d+3)^{\frac{3}{2}}$

For the expected norm, we have
    \begin{align*}
        &\quad\; \EB[\norm{\hat{\nabla} f(x, u) - \nabla f(x)}] \\
        &\stackrel{\eqref{eq:near_fux}}{=} \EB[\norm{uu^\top \nabla f(x)+ \frac{Lv}{2}s(x, u)\norm{u}u - \nabla f(x)}] \\
        &\stackrel{\eqref{eq:proj_sum_lemma}}{\leq} 2\EB[\norm{(uu^\top - I)\nabla f(x)}] + \frac{L^2v^2}{2}\EB[\Vert u\Vert^6] \\
        &\stackrel{\eqref{eq:proj_gaussian_lemma2}}{\leq} 2(d+1) \norm{\nabla f(x)} + \frac{L^2v^2}{2}\EB[\Vert u\Vert^6],
    \end{align*}
    which implies the result.
\end{proof}

\begin{proof}[\textbf{Proof of Lemma \ref{lemma:nearly_gk}.}]
From Eq.\eqref{eq:g_definition} and Eq.\eqref{eq:tdfw_definition}, we have
\begin{align*}
    &\quad\; \EB_{u, u_k, i}[g_k] \\
    &= \EB_{u, u_k, i}[\hat{\nabla} f_i(x_k, u_k) - \hat{\nabla} f_i(w_k, u_k) + \td{h}_k + \hat{\nabla} f(w_k, u) - uu^\top \td{h}_k]\\
    &\leftstackrel{\eqref{eq:near_fiux}, \eqref{eq:near_fux}}{=} \EB_{u, u_k, i}\Big[u_ku_k^\top \nabla f_i(x_k) + \frac{Lv}{2}s_i(x_k, u_k) \norm{u_k}u_k - u_ku_k^\top \nabla f_i(w_k) - \frac{Lv}{2}s_i(w_k, u_k)\norm{u_k}u_k + \\
        &\qquad \td{h}_k + uu^\top \nabla f(w_k) + \frac{Lv}{2}s(w_k, u)\norm{u}u - uu^\top \td{h}_k\Big]\\
    &= \nabla f(x_k) + \frac{Lv}{2}\EB_{u, u_k, i}\Big[(s_i(x_k, u_k) - s_i(w_k, u_k))\norm{u_k}u_k + s(w_k, u)\norm{u} u\Big].
\end{align*}

Let $\tau_{i, k} = \EB_{u, u_k, i}[(s_i(x_k, u_k) - s_i(w_k, u_k))\norm{u_k}u_k + s(w_k, u)\norm{u} u]$. Then, the upper bound for the norm of $\tau_{i, k}$ is 
\begin{align*}
    \Vert \tau_{i, k}\Vert &= \Big\Vert \EB_{u, u_k, i}[(s_i(x_k, u_k) - s_i(w_k, u_k))\norm{u_k}u_k + s(w_k, u)\norm{u} u]\Big\Vert\\
    &\leq \Big\Vert \EB[(s_i(x_k, u_k) - s_i(w_k, u_k))\norm{u_k}u_k]\Big\Vert + \Big\Vert \EB[s(w_k, u)\norm{u}u]\Big\Vert\\
    &\leq \EB\Big[\Big\Vert (s_i(x_k, u_k) - s_i(w_k, u_k))\norm{u_k}u_k\Big\Vert\Big] + \EB\Big[\Big\Vert s(w_k, u)\norm{u}u\Big\Vert\Big]\\
    &\leq \EB[\Vert u_k\Vert^3] + \EB[\Vert u\Vert^3].
\end{align*}
Finally, combining with Lemma \ref{lemma:bound_random_norm}, we can conclude the result.
\end{proof}

\begin{proof}[\textbf{Proof of \cref{lemma:recursion_of_1}}] According to the convexity of $\psi(x)$, the optimal point $x^*$ satisfies 
\begin{align*}
    x^* = \prox(x^* - \eta \nabla f(x^*)).
\end{align*}

Hence, based on Eq.\eqref{eq:update_x}, we have
\begin{align*}
    &\quad\; \EB[\norm{x_{k+1} - x^*}] \\
    &= \EB[\prox(x_k - \eta g_k) - \prox(x^* - \eta \nabla f(x^*))]\\
    &\leq \EB[\norm{x_k - x^* - \eta(g_k - \nabla f(x^*))}]\\
    &= \norm{x_k - x^*} - 2\eta \EB[\dotprod{x_k - x^*, g_k - \nabla f(x^*)}] + \eta^2\EB[\norm{g_k - \nabla f(x^*)}]\\
    &\stackrel{\eqref{eq:expectation_of_gk}}{=} \norm{x_k - x^*} - 2\eta\dotprod{x_k - x^*, \nabla f(x_k) - \nabla f(x^*)} - Lv\eta\dotprod{x_k-x^*, \tau_{i, k}} + \eta^2\EB[\norm{g_k - \nabla f(x^*)}]\\
    &\leq (1-\eta \mu)\norm{x_k - x^*} - 2\eta(f(x_k) - f(x^*) - \dotprod{\nabla f(x^*), x_k- x^*}) - Lv\eta\dotprod{x_k - x^*, \tau_{i, k}} + \\
        &\qquad \eta^2\EB[\norm{g_k - \nabla f(x^*)}]\\
    &\stackrel{\eqref{eq:proj_young}}{\leq}(1-\frac{\mu}{2}\eta)\norm{x_k - x^*} - 2\eta(f(x_k) - f(x^*) - \dotprod{\nabla f(x^*), x_k- x^*}) + \frac{L^2v^2\eta}{2\mu}\norm{\tau_{i, k}} + \eta^2\EB[\norm{g_k - \nabla f(x^*)}].
\end{align*}
Coupling with Lemma \ref{lemma:nearly_gk}, it yields the final iterative formula.
\end{proof}

\begin{proof}[\textbf{Proof of \cref{lemma:recursion_of_2}}]
    According to the update of $\td{h}$, we obtain 
    \begin{align*}
        &\quad\; \EB[\norm{\td{h}_{k+1} - \nabla f(x^*)}]\\
        &= p\EB[\norm{\td{h}_k + \frac{1}{d+2}(\hn f(x_k, u) - uu^\top \td{h}_k) -\nabla f(x^*)}] + (1-p) \norm{\td{h}_k - \nabla f(x^*)}\\
        &\stackrel{\eqref{eq:directional_derivative}}{=}p\underbrace{\EB\Big[\Big\Vert \td{h}_k + \frac{1}{d+2}(uu^\top \nabla f(x_k) +\frac{Lv}{2}s(x_k, u)\norm{u}u- uu^\top\td{h}_k) -\nabla f(x^*)\Big\Vert^2 \Big]}_{A_1} + (1-p) \norm{\td{h}_k - \nabla f(x^*)}.
    \end{align*}

    For $A_1$, we have
    \begin{align*}
            A_1
        &= \norm{\td{h}_k - \nabla f(x^*)} + \frac{2}{d+2}\EB\Big[\dotprod{\td{h}_k - \nabla f(x^*), uu^\top(\nabla f(x_k) - \td{h}_k)+ \frac{Lv}{2}s(x_k, u)\norm{u}u}\Big] + \\
            &\qquad \underbrace{\frac{1}{(d+2)^2}\EB\Big[\bnorm{uu^\top(\nabla f(x_k) - \td{h}_k )+ \frac{Lv}{2}s(x_k, u)\norm{u}u}]}_{A_2}\\
        &= \norm{\td{h}_k - \nabla f(x^*)} + \frac{2}{d+2}\dotprod{\td{h}_k - \nabla f(x^*), \nabla f(x_k) - \nabla f(x^*) + \nabla f(x^*) - \td{h}_k} + \\
            &\qquad \frac{Lv}{d+2}\EB\Big[\dotprod{\td{h}_k - \nabla f(x^*), s(x_k, u)\norm{u}u}\Big] + A_2\\
        &= (1-\frac{2}{d+2})\norm{\td{h}_k - \nabla f(x^*)} + \frac{2}{d+2}\dotprod{\td{h}_k - \nabla f(x^*), \nabla f(x_k) - \nabla f(x^*)} + \\
            &\qquad \frac{Lv}{d+2}\EB\Big[\dotprod{\td{h}_k - \nabla f(x^*), s(x_k, u)\norm{u}u}\Big] + A_2 \\
        &\stackrel{\eqref{eq:proj_young}}{\leq} (1-\frac{2}{d+2})\norm{\td{h}_k - \nabla f(x^*)} + \frac{2}{d+2}\dotprod{\td{h}_k - \nabla f(x^*), \nabla f(x_k) - \nabla f(x^*)} + \\
            &\qquad \frac{1}{6(d+2)}\norm{\td{h}_k -\nabla f(x^*)} + \frac{3L^2v^2}{2(d+2)}\EB[\Vert u\Vert^6] + A_2\\
        &= (1-\frac{11}{6(d+2)})\norm{\td{h}_k - \nabla f(x^*)} + \frac{2}{d+2}\dotprod{\td{h}_k - \nabla f(x^*), \nabla f(x_k) - \nabla f(x^*)} + \frac{3L^2v^2}{2(d+2)}\EB[\Vert u\Vert^6] + A_2.
    \end{align*}
    
    For $A_2$, we have
    \begin{align*}
        A_2 &= \frac{1}{(d+2)^2}\EB\Big[\Big\Vert uu^\top(\nabla f(x_k) - \td{h}_k) + \frac{Lv}{2}s(x_k, u)\norm{u}u\Big\Vert^2\Big]\\
        &\stackrel{\eqref{eq:proj_young}}{\leq} \frac{1}{(d+2)^2}\EB[\frac{7}{6}\norm{uu^\top (\nabla f(x_k) - \td{h}_k)} + \frac{7L^2v^2}{4}s^2(x_k, u)\Vert u\Vert ^6]\\
        &\stackrel{\eqref{eq:proj_gaussian_lemma2}}{\leq} \frac{7}{6(d+2)}\norm{\nabla f(x_k) - \td{h}_k} + \frac{7L^2v^2}{4(d+2)^2}\EB[\Vert u\Vert^6]\\
        &= \frac{7}{6(d+2)}\Big[\norm{\nabla f(x_k) - \nabla f(x^*)} + 2\dotprod{\nabla f(x^*) - \td{h}_k, \nabla f(x_k) - \nabla f(x^*)} + \norm{\td{h}_k - \nabla f(x^*)}\Big] + \\
            &\qquad\frac{7L^2v^2}{4(d+2)^2}\EB[\Vert u\Vert^6].
    \end{align*}
    
    Plugging $A_2$ into $A_1$, we obtain
    \begin{align*}
        A_1 &\leq (1-\frac{2}{3(d+2)})\norm{\td{h}_k - \nabla f(x^*)} + \frac{1}{3(d+2)}\dotprod{\nabla f(x^*) - \td{h}_k, \nabla f(x_k) - \nabla f(x^*)} +\\
            &\qquad \frac{7}{6(d+2)}\norm{\nabla f(x_k) - \nabla f(x^*)} + \frac{L^2v^2}{2(d+2)}\EB[\Vert u\Vert^6](3 + \frac{7}{2(d+2)})\\
            &\stackrel{\eqref{eq:proj_young}}{\leq }{1- \frac{1}{2(d+2)}}\norm{\td{h}_k - \nabla f(x^*)} + \frac{4}{3(d+2)}\norm{\nabla f(x_k) - \nabla f(x^*)} + \frac{L^2v^2}{2(d+2)}\EB[\Vert u\Vert^6](3+\frac{7}{2(d+2)})\\
            &\leftstackrel{\eqref{eq:proj_smooth}, \eqref{eq:bound_random_norm}}{\leq}(1- \frac{1}{2(d+2)})\norm{\td{h}_k - \nabla f(x^*)} + \frac{8L}{3(d+2)}(f(x_k) - f(x^*) - \dotprod{\nabla f(x^*), x_k - x^*}) + \frac{5(d+6)^3L^2v^2}{2(d+2)},
    \end{align*}
    where we also use the fact that $\frac{7}{2(d+2)}< 2$ for $d\geq 1$ in the last inequality.

    Hence, plugging $A_1$ into $\EB[\norm{\td{h}_{k+1} - \nabla f(x^*)}]$, we obtain that 
    \begin{align*}
        &\quad\; \EB[\norm{\td{h}_{k+1} - \nabla f(x^*)}] \\
        &\leq p\Big[(1- \frac{1}{2(d+2)})\norm{\td{h}_k - \nabla f(x^*)} + \frac{8L}{3(d+2)}(f(x_k) - f(x^*) - \dotprod{\nabla f(x^*), x_k - x^*}) + \frac{5(d+6)^3L^2v^2}{2(d+2)}\Big] + \\
            &\qquad (1-p)\norm{\td{h}_k - \nabla f(x^*)}\\
        &= (1-\frac{p}{2(d+2)})\norm{\td{h}_k - \nabla f(x^*)} + \frac{8Lp}{3(d+2)}(f(x_k) - f(x^*) - \dotprod{\nabla f(x^*), x_k - x^*}) + \frac{5(d+6)^3L^2v^2p}{2(d+2)},
    \end{align*}
    which concludes the result.
\end{proof}

\begin{proof}[\textbf{Proof of Lemma \ref{lemma:recursion_of_3}}]
    According to the update of $w$, we have the following
    \begin{align*}
        &\quad\; \EB[\frac{1}{n}\sum_{i=1}^n \norm{\nabla f_i(w_{k+1}) -\nabla f_i(x^*)}] \\
        &= \frac{1-p}{n}\sum_{i=1}^n \norm{\nabla f_i(w_k) - \nabla f_i(x^*)} + \frac{p}{n}\sum_{i=1}^n \norm{\nabla f_i(x_k) - \nabla f(x^*)}\\
        &\stackrel{\eqref{eq:proj_smooth}}{\leq}  \frac{1-p}{n}\sum_{i=1}^n \norm{\nabla f_i(w_k) - \nabla f_i(x^*)} + \frac{2Lp}{n}\sum_{i=1}^n(f_i(x_k) - f_i(x^*) - \dotprod{\nabla f_i(x^*), x_k - x^*})\\
        &=\frac{1-p}{n}\sum_{i=1}^n \norm{\nabla f_i(w_k) - \nabla f_i(x^*)} + 2Lp(f(x_k) - f(x^*) - \dotprod{\nabla f(x^*), x_k - x^*}).
    \end{align*}
\end{proof}

\begin{proof}[\textbf{Proof of Lema \ref{lemma:gf_relationship}}]
    Taking the expectation w.r.t. $u$, $u_k$, and $i$, we have
    \begin{align*}
        &\quad\; \EB[\norm{g_k - \nabla f(x^*)}]\\
        &\leftstackrel{\eqref{eq:tdfw_definition},\eqref{eq:g_definition}}{=} \EB\Big[\norm{\hn f_i(x_k, u_k) - \hn f_i(w_k, u_k) + \td{h}_k + \hn f(w_k, u) - uu^\top \td{h}_k - \nabla f(x^*)}\Big]\\
        &= \norm{\td{h}_k - \nabla f(x^*)} + 2 \EB\Big[\dotprod{\td{h}_k - \nabla f(x^*), \hn f_i(x_k, u_k) - \hn f_i(w_k, u_k) + \hn f(w_k, u) - uu^\top \td{h}_k}\Big] + \\
            &\qquad \underbrace{\EB\Big[\norm{\hn f_i(x_k, u_k) - \hn f_i(w_k, u_k) + \hn f(w_k, u) - uu^\top \td{h}_k}\Big]}_{B_1}\\
        &\stackrel{\eqref{eq:expectation_of_gk}}{=} \norm{\td{h}_k - \nabla f(x^*)} + 2 \dotprod{\td{h}_k - \nabla f(x^*), \nabla f(x_k)-\nabla f(x^*) + \nabla f(x^*) - \td{h}_k} + Lv\dotprod{\td{h}_k - \nabla f(x^*), \tau_{i,k}} + B_1\\
        &=-\norm{\td{h}_k - \nabla f(x^*)} + 2\dotprod{\td{h}_k - \nabla f(x^*), \nabla f(x_k)-\nabla f(x^*)} +  Lv\dotprod{\td{h}_k - \nabla f(x^*), \tau_{i,k}} + B_1.
    \end{align*}

    For $B_1$, with Eq.\eqref{eq:stochastic_directional_derivative} and Eq.\eqref{eq:directional_derivative}, we have
    \begin{align*}
        B_1 &= \EB[\norm{\hn f_i(x_k, u_k) - \hn f_i(w_k, u_k) + \hn f(w_k, u) - uu^\top \td{h}_k}]\\
        &\leftstackrel{\eqref{eq:stochastic_directional_derivative}, \eqref{eq:directional_derivative}}{=} \EB\Big[\bnorm{u_ku_k^\top \nabla f_i(x_k) + \frac{Lv}{2}s_i(x_k, u_k)\norm{u_k}u_k - u_ku_k^\top \nabla f_i(w_k) - \frac{Lv}{2}s_i(w_k, u_k)\norm{u_k}u_k + uu^\top \nabla f(w_k)+ \\
            &\qquad \frac{Lv}{2}s(w_k, u)\norm{u}u - uu^\top \td{h}_k}\Big]\\
        &\stackrel{\eqref{eq:proj_young}}{\leq} 2 \underbrace{\EB\Big[\bnorm{u_ku_k^\top \nabla f_i(x_k) - u_ku_k^\top \nabla f_i(w_k) +uu^\top \nabla f(w_k) - uu^\top \td{h}_k}\Big]}_{B_2} + \\
            &\qquad 2\underbrace{\EB\Big[\bnorm{\frac{Lv}{2}(s_i(x_k, u_k) - s_i(w_k, u_k))\norm{u_k}u_k + \frac{Lv}{2}s(w_k, u)\norm{u}u}\Big]}_{B_3}.
    \end{align*}

    For $B_2$, we have that 
    \begin{align*}
        B_2 &\stackrel{\eqref{eq:proj_gaussian_lemma2}}{=} (d+2)\EB[\norm{\nabla f_i(x_k) - \nabla f_i(x^*) + \nabla f_i(x^*) - \nabla f_i(w_k)}] + (d+2)\norm{\nabla f(w_k) - \td{h}_k} + \\
            &\qquad 2\dotprod{\nabla f(x_k) - \nabla f(w_k), \nabla f(w_k) - \td{h}_k}\\
        &\stackrel{\eqref{eq:proj_young}}{\leq} \frac{2(d+2)}{n}\sum_{i=1}^n \norm{\nabla f_i(x_k) - \nabla f_i(x^*)} + \frac{2(d+2)}{n}\sum_{i=1}^n \norm{\nabla f_i(w_k) - \nabla f_i(x^*)} + (d+2)\norm{\nabla f(w_k) - \td{h}_k} + \\
            &\qquad 2\dotprod{\nabla f(x_k) - \nabla f(w_k), \nabla f(w_k) - \td{h}_k}\\
        &=\frac{2(d+2)}{n}\sum_{i=1}^n \norm{\nabla f_i(x_k) - \nabla f_i(x^*)} + \frac{2(d+2)}{n}\sum_{i=1}^n \norm{\nabla f_i(w_k) - \nabla f_i(x^*)} + (d+2)\norm{\nabla f(w_k) - \nabla f(x^*)} + \\
            &\qquad (d+2)\norm{\nabla f(x^*) - \td{h}_k} + 2(d+2)\dotprod{\nabla f(w_k) - \nabla f(x^*), \nabla f(x^*) - \td{h}_k} + \\
            &\qquad 2\dotprod{\nabla f(x_k) - \nabla f(w_k), \nabla f(w_k) - \td{h}_k}.
    \end{align*}
    Since 
    \begin{align*}
        &\quad\; 2\dotprod{\nabla f(x_k) - \nabla f(w_k), \nabla f(w_k) - \td{h}_k} \\
        &=2\dotprod{\nabla f(x_k) - \nabla f(x^*), \nabla f(w_k) - \nabla f(x^*)} + 2\dotprod{\nabla f(x_k) - \nabla f(x^*), \nabla f(x^*) - \td{h}_k} - 2\norm{f(w_k) - f(x^*)} + \\
            &\qquad 2\dotprod{\nabla f(x^*) - \nabla f(w_k), \nabla f(x^*) - \td{h}_k},
    \end{align*}
    plugging it into $B_2$, we obtain that
    \begin{align*}
        B_2 &\leq \frac{2(d+2)}{n}\sum_{i=1}^n \norm{\nabla f_i(x_k) - \nabla f_i(x^*)} + \frac{2(d+2)}{n}\sum_{i=1}^n \norm{\nabla f_i(w_k) - \nabla f_i(x^*)} + d\norm{\nabla f(w_k) - \nabla f(x^*)} + \\
            &\qquad (d+2)\norm{\nabla f(x^*) - \td{h}_k} + 2(d+1)\dotprod{\nabla f(w_k) - \nabla f(x^*), \nabla f(x^*) - \td{h}_k} + \\
            &\qquad 2\dotprod{\nabla f(x_k) - \nabla f(x^*), \nabla f(w_k) - \nabla f(x^*)} + 2\dotprod{\nabla f(x_k) - \nabla f(x^*), \nabla f(x^*) - \td{h}_k}.
    \end{align*}
    With Lemma \ref{lemma:proj_sum_lemma} and \ref{lemma:proj_Young}, we have 
    \begin{align*}
        2(d+1)\dotprod{\nabla f(w_k) - \nabla f(x^*), \nabla f(x^*) - \td{h}_k} &\stackrel{\eqref{eq:proj_young}}{\leq} (d+1)\norm{\nabla f(w_k) - \nabla f(x^*)} + (d+1)\norm{\nabla f(x^*) - \td{h}_k},\\
        2\dotprod{\nabla f(x_k) - \nabla f(x^*), \nabla f(w_k) - \nabla f(x^*)}  &\stackrel{\eqref{eq:proj_young}}{\leq} 
        \norm{\nabla f(x_k) - \nabla f(x^*)} + \norm{\nabla f(w_k) - \nabla f(x^*)},\\
        \norm{\nabla f(w_k) - \nabla f(x^*)} &\stackrel{\eqref{eq:proj_sum_lemma}}{\leq} \frac{1}{n}\sum_{i=1}^n \norm{\nabla f_i(w_k) - \nabla f(x^*)}.
    \end{align*}

    Hence, we have 
    \begin{align*}
        B_2 &\leq \frac{2(d+2)}{n}\sum_{i=1}^n \norm{\nabla f_i(x_k) - \nabla f_i(x^*)} + \frac{2(2d+3)}{n}\sum_{i=1}^n \norm{\nabla f_i(w_k) -\nabla f_i(x^*)} + \\
                &\qquad (2d+3)\norm{\td{h}_k - \nabla f(x^*)} + \norm{\nabla f(x_k) - \nabla f(x^*)} + 2\dotprod{\nabla f(x_k) - \nabla f(x^*), \nabla f(x^*) - \td{h}_k}\\
            &\stackrel{\eqref{eq:proj_sum_lemma}}{\leq} \frac{2d + 5}{n}\sum_{i=1}^n \norm{\nabla f_i(x_k) - \nabla f_i(x^*)} + \frac{2(2d+3)}{n}\sum_{i=1}^n \norm{\nabla f_i(w_k) -\nabla f_i(x^*)} + \\
                &\qquad (2d+3)\norm{\td{h}_k - \nabla f(x^*)} + 2\dotprod{\nabla f(x_k) - \nabla f(x^*), \nabla f(x^*) - \td{h}_k}.
    \end{align*}

    For $B_3$, we have
    \begin{align*}
        B_3 &= \frac{L^2v^2}{4}\EB\Big[\bnorm{(s_i(x_k, u_k) - s_i(w_k, u_k))\norm{u_k}u_k + s(w_k, u)\norm{u}u}\Big]\\
        &\stackrel{\eqref{eq:proj_young}}{\leq} \frac{L^2v^2}{2} (\EB[\Vert u_k\Vert^6] + \EB[\Vert u\Vert^6])\\
        &\stackrel{\eqref{eq:bound_random_norm}}{\leq} L^2v^2(d+6)^3.
    \end{align*}

    Plugging $B_2$ and $B_3$ into $B_1$, we obtain 
    \begin{align*}
        B_1 &\leq \frac{2(2d + 5)}{n}\sum_{i=1}^n \norm{\nabla f_i(x_k) - \nabla f_i(x^*)} + \frac{4(2d+3)}{n}\sum_{i=1}^n \norm{\nabla f_i(w_k) -\nabla f_i(x^*)} + \\
                &\qquad 2(2d+3)\norm{\td{h}_k - \nabla f(x^*)} + 4\dotprod{\nabla f(x_k) - \nabla f(x^*), \nabla f(x^*) - \td{h}_k} + 2L^2v^2(d+6)^3.
    \end{align*}

    Plugging $B_1$ into the original inequality, we have
    \begin{align*}
        &\quad\; \EB[\norm{g_k - \nabla f(x^*)}]\\
        &\leq (4d+5)\norm{\td{h}_k - \nabla f(x^*)} + 2\dotprod{\nabla f(x_k) - \nabla f(x^*), \nabla f(x^*) - \td{h}_k} + Lv\dotprod{\td{h}_k - \nabla f(x^*), \tau_{i, k}} + \\
            &\qquad \frac{2(2d + 5)}{n}\sum_{i=1}^n \norm{\nabla f_i(x_k) - \nabla f_i(x^*)} + \frac{4(2d+3)}{n}\sum_{i=1}^n \norm{\nabla f_i(w_k) -\nabla f_i(x^*)} + 2L^2v^2(d+6)^3\\
        &\leftstackrel{\eqref{eq:proj_young}, \eqref{eq:proj_sum_lemma}}{\leq} 2(2d+3)\norm{\td{h}_k - \nabla f(x^*)} + \frac{L^2v^2}{2}\norm{\tau_{i, k}} + \frac{4(d+3)}{n}\sum_{i=1}^n \norm{\nabla f_i(x_k) - \nabla f_i(x^*)} + \\
            &\qquad \frac{4(2d+3)}{n}\sum_{i=1}^n \norm{\nabla f_i(w_k) -\nabla f_i(x^*)} + 2L^2v^2(d+6)^3\\
        &\leftstackrel{\eqref{eq:proj_smooth}}{\leq} 2(2d+3)\norm{\td{h}_k - \nabla f(x^*)} + \frac{L^2v^2}{2}\norm{\tau_{i, k}} + 4(d+3)L(f(x_k) - f(x^*) - \dotprod{\nabla f(x^*), x_k -x^*}) + \\
            &\qquad \frac{4(2d+3)}{n}\sum_{i=1}^n \norm{\nabla f_i(w_k) -\nabla f_i(x^*)} + 2L^2v^2(d+6)^3,
    \end{align*}
    where we can derive the result with the bound of $\EB[\norm{\tau_{i, k}}]$ in Lemma \ref{lemma:nearly_gk}.
\end{proof}

\begin{proof}[\textbf{Proof of \cref{coro:psi_recursion}}]
Let $C = f(x_k) - f(x^*) - \dotprod{\nabla f(x^*), x_k - x^*}$. With Lemma \ref{lemma:recursion_of_1}, \ref{lemma:recursion_of_2}, \ref{lemma:recursion_of_3}, and $p=\frac{1}{n}$, we have 
\begin{align*}
    &\quad\; \EB[\Psi(x_{k+1})]\\
    &\leftstackrel{\eqref{eq:recursion_of_1}, \eqref{eq:recursion_of_2}, \eqref{eq:recursion_of_3}}{\leq} (1-\frac{\mu}{2}\eta)\norm{x_k - x^*} - 2\eta C + \frac{2(d+3)^3L^2v^2\eta}{\mu} + \eta^2 \EB[\norm{g_k - \nabla f(x^*)}] +  \alpha \Big[\frac{5(d+6)^3L^2v^2}{2n(d+2)} + \\
        &\qquad(1-\frac{1}{2n(d+2)})\norm{\td{h}_k - \nabla f(x^*)} + \frac{8L}{3n(d+2)}C \Big] + \beta \Big[(1-\frac{1}{n})\frac{1}{n}\sum_{i=1}^n \norm{\nabla f_i(w_k) - \nabla f_i(x^*)} + \frac{2L}{n}C \Big].
\end{align*}   

Plugging Lemma \ref{lemma:gf_relationship} into the above inequality, we achieve 
\begin{align*}
    &\quad\; \EB[\psi(x_{k+1})] \\
    &\stackrel{\eqref{eq:gf_relationship}}{\leq} (1-\frac{\mu}{2}\eta)\norm{x_k - x^*} - 2\eta C + \frac{2(d+3)^3L^2v^2\eta}{\mu} + \eta^2\Big[ 2(2d+3)\norm{\td{h}_k - \nabla f(x^*)} + 4L^2v^2(d+6)^3 + \\
        &\qquad 4(d+3)LC + \frac{4(2d+3)}{n}\sum_{i=1}^n \norm{\nabla f_i(w_k) -\nabla f_i(x^*)}\Big] + \alpha \Big[(1-\frac{1}{2n(d+2)})\norm{\td{h}_k - \nabla f(x^*)} + \frac{8L}{3n(d+2)}C + \\
        &\qquad \frac{5(d+6)^3L^2v^2}{2n(d+2)}\Big] + \beta \Big[(1-\frac{1}{n})\frac{1}{n}\sum_{i=1}^n \norm{\nabla f_i(w_k) - \nabla f_i(x^*)} + \frac{2L}{n}C \Big].
\end{align*}
Rearranging the preceding inequality, we obtain
\begin{align*}
    &\quad\; \EB[\psi(x_{k+1})] \\
    &\leq (1-\frac{\mu}{2}\eta)\norm{x_k - x^*} + \alpha(1-\frac{1}{2n(d+2)} + \frac{2(2d+3)\eta^2}{\alpha})\norm{\td{h}_k - \nabla f(x^*)} + \\
        &\qquad \beta(1-\frac{1}{n} + \frac{4(2d+3)\eta^2}{\beta})\frac{1}{n}\sum_{i=1}^n \norm{\nabla f_i(w_k) - \nabla f_i(x^*)} - (2\eta - 4(d+3)L\eta^2 - \frac{8L}{3n(d+2)}\alpha - \frac{2L}{n}\beta)C + \\
        &\qquad  \frac{2(d+3)^3L^2v^2\eta}{\mu} + 4L^2v^2(d+6)^3\eta^2 + \frac{5(d+6)^3L^2v^2}{2n(d+2)}\alpha,
\end{align*}
which concludes the result.
\end{proof}

\begin{proof}[\textbf{Proof of \cref{thm:psi_convergence}}]
Plugging the values of $\eta$, $\alpha$, $\beta$ in to \cref{coro:psi_recursion}, we obtain that 
\begin{align*}
    &\quad\; \EB[\Psi(x_{k+1})] \\
    &\leq (1-\frac{1}{(80d+126)\kappa})\norm{x_{k} - x^*} + (1-\frac{1}{4n(d+2)})\alpha\norm{\td{h}_k - \nabla f(x^*)} + \\
        &\qquad \beta(1-\frac{1}{2n})\frac{1}{n}\sum_{i=1}^n \norm{\nabla f_i(w_k) - \nabla f_i(x^*)} -\lambda C + \frac{2(d+3)^3v^2\kappa}{40d + 63} + \frac{8(5d+8)(d+6)^3v^2}{(40d+63)^2},
\end{align*}
where $\lambda = \frac{2}{(40d+63)L} - \frac{\frac{188}{3}d + 100}{(40d+63)^2L}$, i.e., $\lambda > 0$. Then, we can eliminate the term $C$ in the above inequality and obtain,
\begin{align*}
    &\quad\; \EB[\Psi(x_{k+1})] \\ 
    &\leq (1-\frac{1}{(80d+126)\kappa})\norm{x_{k} - x^*} + (1-\frac{1}{4n(d+2)})\alpha\norm{\td{h}_k - \nabla f(x^*)} + \\
        &\qquad \beta(1-\frac{1}{2n})\frac{1}{n}\sum_{i=1}^n \norm{\nabla f_i(w_k) - \nabla f_i(x^*)} + \delta, \\
    &\leq \max\{1-\frac{1}{\kappa(80d+126)}, 1-\frac{1}{4n(d+2)}\} \Psi(x_k) + \delta,
\end{align*}
where $\delta = \frac{2(d+3)^3v^2\kappa}{40d + 63} + \frac{8(5d+8)(d+6)^3v^2}{(40d+63)^2}$.
\end{proof}

\begin{proof}[\textbf{Proof of \cref{coro:convergence_property}}]
Since $\theta = \frac{1}{\kappa(80d+126) + 4n(d+2)} \leq \min \{\frac{1}{\kappa(80d+126)}, \frac{1}{4n(d+2)}\}$, based on \cref{thm:psi_convergence}, we have that 
\begin{align*}
    \EB[\Psi(x_{k+1})] \leq (1-\theta) \Psi(x_k) + \delta.
\end{align*}

Telescoping it from $k=0$ for $K-1$, we obtain that 
\begin{align*}
    \EB[\Psi(x_K)] &\leq (1-\theta)^K\Psi(x_0) + \sum_{k=0}^{k-1} (1-\theta)^k \delta\\
    &= (1-\theta)^K\Psi(x_0) + \frac{1-(1-\theta)^K}{\theta}\delta\\
    &\leq (1-\theta)^K\Psi(x_0) + \frac{\delta}{\theta}.
\end{align*}
Furthermore, since $\delta \leq (\kappa + 1)d^2v^2$, we conclude the result.
\end{proof}

\section{Hyperparameter Tuning}
\label{sec:hyperparameter}
We employ the grid search method to determine the optimal hyperparameters for \texttt{ZPDVR} and other baseline methods. 
The total number of SZO function evaluations, and the batch size of samples and directions are fixed and equal for all methods. 
Apart from the learning rate $\eta$, the probability $\rho$ in \texttt{ZPDVR} and the number of the inner loop $m$ in \texttt{ZPSVRG} are also tuned parameters. 
The hyperparameter search fields for each method in the four data sets are as follows:
\begin{itemize}
    \item a9a. The learning rate $\eta$ is ranged from $\{1\times 10^{-2}, 5\times 10^{-2}, 1\times 10^{-1}, 5\times 10^{-1}, 1, 2, 5, 10\}$ for \texttt{PGD}, $\{1\times 10^{-1}, 5\times 10^{-1}, 1, 5, 10\}$ for \texttt{ZPDVR}, $\{1\times 10^{-3}, 5\times 10^{-3}, 1\times 10^{-2}, 5\times 10^{-2}, 1\times 10^{-1}, 5\times 10^{-1}, 1\}$ for \texttt{ZPSVRG}, $\{1\times 10^{-3}, 5\times 10^{-3}, 1\times 10^{-2}, 5\times 10^{-2}, 1\times 10^{-1}, 5\times 10^{-1}, 1, 5, 10\}$ for \texttt{SEGA}.
    The probability $\rho$ is ranged from $\{1\times 10^{-2}, 2\times 10^{-2}, 3\times 10^{-2}, 4\times 10^{-2}, 5\times 10^{-2}\}$ and the number of inner loop $m$ is selected from $\{10, 50, 100, 500, 1000, 5000\}$;
    \item w8a. The learning rate $\eta$ is ranged from $\{1\times 10^{-1}, 5\times 10^{-1}, 1, 2, 5, 10\}$ for \texttt{PGD}, $\{1\times 10^{-3}, 7\times 10^{-3}, 1\times 10^{-2}, 7\times 10^{-2}, 1\times 10^{-1}\}$ for \texttt{ZPDVR}, $\{1\times 10^{-3}, 5\times 10^{-3}, 7\times 10^{-3}, 1\times 10^{-2}, 5\times 10^{-2}, 7\times 10^{-2}, 1\times 10^{-1}\}$ for \texttt{ZPSVRG}, $\{1\times 10^{-3}, 5\times 10^{-3}, 1\times 10^{-2}, 5\times 10^{-2}, 1\times 10^{-1}, 5\times 10^{-1}, 1, 5, 10\}$ for \texttt{SEGA}.
    The probability $\rho$ is ranged from $\{5\times 10^{-4}, 1\times 10^{-3}, 7\times 10^{-3}\}$ and the number of inner loop $m$ is selected from $\{10, 50, 100, 500, 1000\}$;
    \item covtype. The Learning rate $\eta$ is ranged from $\{1\times 10^{-2}, 5\times 10^{-2}, 1\times 10^{-1}, 5\times 10^{-1}, 1, 2, 3, 4, 5, 10\}$ for \texttt{PGD}, $\{1\times 10^{-3}, 5\times 10^{-3}, 1\times 10^{-2}, 5\times 10^{-2}, 1\times 10^{-1}, 5\times 10^{-1}, 1\}$ for \texttt{ZPDVR}, $\{1\times 10^{-3}, 5\times 10^{-3}, 1\times 10^{-2}, 2\times 10^{-2}, 5\times 10^{-2}, 1\times 10^{-1}\}$ for \texttt{ZPSVRG},
    $\{1\times 10^{-3}, 5\times 10^{-3}, 1\times 10^{-2}, 5\times 10^{-2}, 1\times 10^{-1}, 5\times 10^{-1}, 1, 5, 10\}$ for \texttt{SEGA}.
    The number of inner loop $m$ is selected from $\{10, 50, 100, 500, 1000\}$.
    In this dataset, we periodically update $w$ and $\td{h}$ for \texttt{ZPDVR} (update these parameters after completing a full pass over the data set). This procedure omits the randomness caused by $\rho$ and is equivalent to set $\rho=\frac{B}{n}$, where $B$ is the batch size of samples;
    \item gisette. The learning rate $\eta$ is range from $\{1\times 10^{-4}, 5\times 10^{-4}, 1\times 10^{-3}, 5\times 10^{-3}, 1\times 10^{-2}, 5\times 10^{-2}, 1\times 10^{-1}\}$ for \texttt{PGD}, $\{1\times 10^{-3}, 5\times 10^{-3}, 1\times 10^{-2}, 5\times 10^{-2}, 1\times 10^{-1}\}$ for \texttt{ZPDVR}, $\{1\times 10^{-4}, 5\times 10^{-4}, 1\times 10^{-3}, 5\times 10^{-3}, 1\times 10^{-2}, 5\times 10^{-2}, 1\times 10^{-1}\}$ for \texttt{ZSVRG}, $\{1\times 10^{-3}, 5\times 10^{-3}, 1\times 10^{-2}, 5\times 10^{-2}, 1\times 10^{-1}, 5\times 10^{-1}, 1, 5, 10\}$ for \texttt{SEGA}.
    The number of inner loop $m$ is selected from $\{10, 20, 50, 100\}$.
    Similarly, we also periodically update $w$ and $\td{h}$ for \texttt{ZPDVR}.
\end{itemize}


\end{document}